\documentclass[a4paper,11pt]{amsart}

\usepackage{a4wide,amsmath,amsfonts,amssymb,amsthm,mathrsfs,bbm}
\usepackage[hyperfootnotes=false]{hyperref}
\usepackage{xcolor}
\usepackage{multirow}
\usepackage{caption}
\usepackage{graphicx, array}
\usepackage{tabularray}
\usepackage{float}

\allowdisplaybreaks[2]

\numberwithin{equation}{section}

\newtheorem{theorem}{Theorem}[section]

\newtheorem{remark}[theorem]{Remark}
\newtheorem{proposition}[theorem]{Proposition}
\newtheorem{assumption}[theorem]{Assumption}

\renewcommand{\d}{\mathrm{d}}
\newcommand{\dd}{\,\mathrm{d}}
\newcommand{\R}{\mathbb{R}}
\newcommand{\N}{\mathbb{N}}
\newcommand{\E}{\mathbb{E}}
\newcommand{\bR}{\mathbb{R}}
\newcommand{\bN}{\mathbb{N}}
\newcommand{\bE}{\mathbb{E}}
\newcommand{\bP}{\mathbb{P}}

\title[Neural stochastic Volterra equations]{Neural stochastic Volterra equations:\\ learning path-dependent dynamics}

\author[Bergerhausen]{Martin Bergerhausen}
\address{Martin Bergerhausen, University of Mannheim, Germany}
\email{martin.bergerhausen@uni-mannheim.de}

\author[Pr{\"o}mel]{David J. Pr{\"o}mel}
\address{David J. Pr{\"o}mel, University of Mannheim, Germany}
\email{proemel@uni-mannheim.de}

\author[Scheffels]{David Scheffels}
\address{David Scheffels, University of Mannheim, Germany}
\email{david.scheffels@lessing-ffm.net}

\date{\today}

\begin{document}
 
\begin{abstract}
  Stochastic Volterra equations (SVEs) serve as mathematical models for the time evolutions of random systems with memory effects and irregular behaviour. We introduce neural stochastic Volterra equations as a physics-inspired architecture, generalizing the class of neural stochastic differential equations, and provide some theoretical foundation. Numerical experiments on various SVEs, like the disturbed pendulum equation, the generalized Ornstein--Uhlenbeck process, the rough Heston model and a monetary reserve dynamics, are presented, comparing the performance of neural SVEs, neural SDEs and Deep Operator Networks (DeepONets).
\end{abstract}

\maketitle





\section{Introduction}

Stochastic Volterra equations (SVEs) are used as mathematical models for the time evolutions of random systems appearing in various areas like biology, finance or physics. SVEs are a natural generalization of ordinary stochastic differential equations (SDEs) and, in contrast to SDEs, they are capable to represent random dynamics with memory effects and very irregular trajectories. For instance, SVEs are used in the modelling of turbulence \cite{Barndorff-Nielsen2008}, of volatility on financial markets~\cite{ElEuch2019} and of deoxyribonucleic acid (DNA) patterns \cite{Reynaud-Bouret2010}.

Combining differential equations and neural networks into hybrid approaches for statistical learning has been gaining increasing interest in recent years, see e.g. \cite{E2017,Chen2018}. This has led to many very successful data-driven methods to learn solutions of various differential equations. For instance, neural stochastic differential equations are SDEs with coefficients parametrized by neural networks, and serve as continuous-time generative models for irregular time series, see \cite{Liu2019,Li2020,Kidger2021,Issa2024}. Models based on neural SDEs are of particular interest in financial engineering, see \cite{Cuchiero2020,Gierjatowicz2022,Cohen2022}. Further examples of `neural' differential equations are neural controlled differential equations \cite{Kidger2020}, which led to very successful methods for irregular time series, neural rough differential equations \cite{Morrill2021}, which are especially well-suited for long time series, and neural stochastic partial differential equations \cite{Salvi2022}, which are capable to process data from continuous spatiotemporal dynamics. Loosely speaking, `neural' differential equations and their variants can be considered as continuous-time analogous to various recurrent neural networks.

In the present work, we introduce \textit{neural stochastic Volterra equations} as stochastic Volterra equations with coefficients parameterized by neural networks. They constitute a natural generalization of neural SDEs with the advantage that they are capable to represent time series with temporal dependency structures, which overcomes a limitation faced by neural SDEs. Hence, neural SVEs are suitable to serve as generative models for random dynamics with memory effects and irregular behaviour, even more irregular than neural SDEs. As theoretical justification for the universality of neural SVEs, we provide a stability result for general SVEs in Proposition~\ref{prop: stability SVE}, which can be combined with classical universal approximation theorems for neural networks~\cite{Cybenko1989,Hornik1991,Kidger2020,Kwossek2025}; cf. Remark~\ref{rem: neural networks}.

Relying on neural stochastic Volterra equations parameterized by feedforward neural networks, we study supervised learning problems for random Volterra type dynamics. More precisely, we consider setups, where the training sets consist of sample paths of the 'true'  Volterra process together with the associated realizations of the driving noise and the initial condition, and build a neural SVEs based model aiming to reproduce the sample paths as good as possible. A related supervised learning problem in the context of stochastic partial differential equations (SPDEs) was treated in~\cite{Salvi2022} introducing neural SPDEs. For unsupervised learning problems using neural SDEs we refer to \cite{Kidger2022}.

We numerically investigate the supervised learning problem for prototypical Volterra type dynamics such as the disturbed pendulum equation \cite{Oksendal2003}, the rough Heston model \cite{ElEuch2019}, the generalized Ornstein--Uhlenbeck process \cite{Vasicek2012} and a model for the dynamics of monetary reserves~\cite{Carmona2018c}. The performance of the neural SVE based models is compared to Deep Operator Networks (DeepONets) and to neural SDEs. Recall DeepONets are a popular class of neural learning algorithms for general operators on function spaces that were introduced in \cite{Lu2021}. For the training process of the neural SVE we choose the Adam algorithm, as introduced in \cite{Kingma2014}, which is known to be a well-suited stochastic gradient descent method for stochastic optimization problem.

The numerical study in Section~\ref{sec:numerical experiments} demonstrates that the presented neural SVE based methods significantly outperform DeepONets; see Table~\ref{table:1}-Table~\ref{table:rh}. In particular, neural SVE based methods generalize much more effectively, as evidenced by their strong performance on the test sets -- neural SVEs are up to $20$ times more accurate than DeepONets. Moreover, neural SVEs also outperform neural SDE based models for random dynamics with dependency structures; cf. Subsection~\ref{subsec: neural SDEs}. These observations highlight the advantages of the physics-informed architecture of neural SVEs for supervised learning problems involving random systems with Volterra-type dynamics.

\smallskip

\noindent \textbf{Organization of the paper:} In Section~\ref{sec:NSVEs} we introduce neural stochastic Volterra equations and their theoretical background. The numerical experiments are presented in Section~\ref{sec:numerical experiments}. In Appendix~\ref{sec: appendix} we present the postponed proofs regarding the stability of stochastic Volterra equations.

\section{Neural stochastic Volterra equations}\label{sec:NSVEs}

Let $(\Omega,\mathcal{F},(\mathcal{F}_t)_{t\in [0,T]},\mathbb{P})$ be a filtered probability space, which satisfies the usual conditions, $T\in (0,\infty)$ and $d,m\in\N$. Given an $\R^d$-valued random initial condition $\xi$ and an $m$-dimensional standard Brownian motion~$(B_t)_{t\in[0,T]}$, we consider the $d$-dimensional \textit{stochastic Volterra equation (SVE)}
\begin{equation}\label{eq:SVE}
  X_t = \xi\,  g(t)+\int_0^t K_{\mu}(t-s)\mu(s,X_s)\dd s+\int_0^t K_{\sigma}(t-s)\sigma(s,X_s)\dd B_s,\quad t\in [0,T],
\end{equation}
where  $g\colon [0,T]\to \R$ is a deterministic continuous function (where we usually normalize $g(0)=1$), the coefficients $\mu\colon [0,T]\times\R^d\to\R^d $ and $\sigma\colon [0,T]\times\R^d\to\R^{d\times m}$, and the convolutional kernels $K_\mu, K_\sigma\colon [0,T]\to \R$ are measurable functions. Furthermore, $\int_0^t K_{\sigma}(t-s)\sigma(s,X_s)\dd B_s$ is defined as an It{\^o} integral. We refer to \cite{Karatzas1991,Oksendal2003} for introductory textbooks on stochastic integration and to \cite{Pardoux1990,Cochran1995,Coutin2001} for classical results on SVEs.

\smallskip

To define the notion of a (strong) $L^p$-solution, let $L^p(\Omega\times [0,T])$ be the space of all real-valued, $p$-integrable functions on $\Omega\times [0,T]$. We call an $(\mathcal{F}_t)_{t\in[0,T]}$-progressively measurable stochastic process $(X_t)_{t\in [0,T]}$ in $L^p(\Omega\times [0,T])$, on the given probability space $(\Omega,\mathcal{F},(\mathcal{F}_t)_{t\in[0,T]},\mathbb{P})$, a (strong) $L^p$-solution of the SVE~\eqref{eq:SVE} if $ \int_0^t (|K_\mu(t-s)\mu(s,X_s)|+|K_\sigma(t-s)\sigma(s,X_s)|^2 )\dd s<\infty$ for all $t\in[0,T]$ and the integral equation~\eqref{eq:SVE} hold $\mathbb{P}$-almost surely. As usual, a strong $L^1$-solution $(X_t)_{t\in [0,T]}$ of the SVE~\eqref{eq:SVE} is often just called solution of the SVE~\eqref{eq:SVE}.

\subsection{Neural SVEs}

To learn the dynamics of the SVE~\eqref{eq:SVE}, that is, the corresponding operators $\xi$, $g$, $K_\mu$, $K_\sigma$, $\mu$ and $\sigma$, we rely on some neural network architecture. To that end, let for some latent dimension $d_h>d$,
\begin{align*}
  &L_\theta\colon \R^d\to\R^{d_h},\quad g_\theta\colon [0,T]\to\R,\quad K_{\mu,\theta}\colon [0,T]\to \R,\quad K_{\sigma,\theta}\colon [0,T]\to \R, \\
  &\mu_\theta\colon [0,T]\times \R^{d_h}\to\R^{d_h}, \quad \sigma_\theta\colon [0,T]\times \R^{d_h}\to\R^{d_h\times m}, \quad \Pi_\theta\colon \R^{d_h}\to \R^d
\end{align*}
be seven feedforward neural networks (see~\cite[Section~3.6.1]{Yadav2015}) that are parameterized by some common parameter $\theta$. Note that $L_\theta$ lifts the given initial value to the latent space $\R^{d_h}$, $\Pi_\theta$ is the readout back from the latent space to the space $\R^d$, and the other networks try to imitate their respectives in Equation \eqref{eq:SVE} on the latent $d_h$-dimensional space.

\smallskip

Given the input data $\xi\in\R^d$ and $(B_t)_{t\in[0,T]}\in C([0,T];\R^m)$, $\mathbb{P}$-a.s., we introduce the \textit{neural stochastic Volterra equations}
\begin{align}\label{eq:neuralSVE}
  Z_0 &= L_\theta(\xi),\notag\\
  Z_t &= Z_0\, g_\theta(t)+\int_0^t K_{\mu,\theta}(t-s)\mu_\theta(s,Z_s)\dd s + \int_0^t K_{\sigma,\theta}(t-s)\sigma_\theta(s,Z_s)\dd B_s,\\
  X_t&=\Pi_\theta(Z_t),\quad t\in [0,T].\notag
\end{align}
The objective is to optimize $\theta$ as good as possible such that the generated paths are as close as possible to the given training paths. Therefore, one needs to solve a stochastic optimization problem at each training step. One typically chosen and well-suited stochastic gradient descent method for stochastic optimization problems is the Adam algorithm, introduced in \cite{Kingma2014}. The Adam algorithm is known to be computationally efficient, requires little memory, is invariant to diagonal rescaling of gradients and is well-suited for high-dimensional problems with regard to data/parameters.

Given a trained supervised model $(L_\theta, K_{\mu,\theta}, K_{\sigma,\theta}, \mu_\theta, \sigma_\theta, \Pi_\theta)$, we can evaluate the neural SVE~\eqref{eq:neuralSVE} given the input data $(\xi,B)$ by using any numerical scheme for stochastic Volterra equations. For that purpose, we use the Volterra Euler--Maruyama scheme introduced in \cite{Zhang2008} for the training procedure. Note that Lipschitz conditions on $\mu_{\theta}$ and $\sigma_\theta$ can be imposed by using, e.g., LipSwish, ReLU or tanh activation functions.

\subsection{Neural network architecture}

The structure of the neural SVE model \eqref{eq:neuralSVE} is analogously defined to the structure of neural stochastic differential equations, as introduced in \cite{Kidger2022}, and of neural stochastic partial differential equations, as introduced in \cite{Salvi2022}. The $d_h$-dimensional process $Z$ represents the hidden state. We impose the readout $\Pi_\theta$ to get back to dimension $d$. The model has, at least if one considers a setting where the initial condition cannot be observed like an unsupervised setting, some minimal amount of architecture. It is in such a setting necessary to induce the lift $L_\theta$ and the randomness by some additional variable $\tilde{\xi}$ to learn the randomness induced by the initial condition $X_0=\Pi_\theta\big(L_\theta(\tilde{\xi})g_\theta(0)\big)$ (otherwise $X_0$ would not be random since it does not depend on the Brownian motion~$B$). Moreover, the structure induced by the lift $L_\theta$ and the readout $\Pi_\theta$ is the natural choice to lift the $d$-dimensional SVE~\eqref{eq:SVE} to the latent dimension $d_h>d$.

\smallskip

We use LipSwish activation functions in any layer of any network. These were introduced in \cite{Chen2019} as $\rho(z)=0.909z\sigma(z)$, where $\sigma$ is the sigmoid function. Due to the constant $0.909$, LipSwish activations are Lipschitz continuous with Lipschitz constant one and smooth. Moreover, there is strong empirical evidence that LipSwish activations are very suitable for a variety of challenging approximation tasks, see~\cite{Ramachandran2017}.

\smallskip

For a given latent dimension $d_h>d$, the lift $L_\theta$ is modeled as a linear $1$-layer network from dimension $d$ to $d_h$ without any additional hidden layer, and, as its counterpart, the readout $\Pi_\theta$ as a linear $1$-layer network from $d_h$ to $d$. The networks $K_{\mu,\theta},K_{\sigma,\theta}$ and $g_\theta$ are all designed as linear networks from dimension $1$ to $1$ with two hidden layers of size $d_K$ for some additional dimension $d_K>d$. Lastly, the network $\mu_\theta$ is defined as a linear network from dimension $1+d_h$ to $d_h$ with one hidden layer of size $d_h$ and the network $\sigma_\theta$ from $1+d_h$ to $d_h\cdot m$ with one hidden layer of size $d_h\cdot m$.

\subsection{Stability for SVEs}

The mathematical reason that neural stochastic Volterra equations provide a suitable structure for learning the dynamics of general SVEs is the universal approximation property of neural networks, see e.g. \cite{Cybenko1989,Hornik1991,Kidger2020b,Kwossek2025}, and the stability result for SVEs, presented in this subsection. More precisely, our stability result yields that if we approximate the kernels and coefficients of an SVE sufficiently well, we get a good approximation of the solution by the respective approximating solutions. To formulate the stability result, we need the following definitions and assumptions.

\smallskip

For $p\geq 1$, the $L^p$-norm of a function $h\colon [0,T]\to\R$ is defined by
\begin{equation*}
  \|h\|_p := \Big( \int_0^T |h(s)|^p\dd s \Big)^{\frac{1}{p}},
\end{equation*}
and the $\sup$-norms for functions $f\colon [0,T]\times \R^d\to \R^d$ and $g\colon [0,T]\to\R^d$, respectively, are given by
\begin{equation*}
  \|f\|_\infty := \sup_{t\in[0,T],x\in\R^d}|f(t,x)| \quad \text{and}\quad \|g\|_\infty := \sup_{t\in[0,T]}|g(t)|.
\end{equation*}

As approximation of the SVE~\eqref{eq:SVE}, we consider a sequence of SVEs, given by
\begin{equation}\label{eq:subsequence of SVEs}
  X^n_t = \xi g_n(t) + \int_0^t K_{\mu,n} (t-s) \mu_n(s,X^n_s) \dd s + \int_0^t K_{\sigma,n} (t-s) \sigma_n(s, X^n_s) \dd B_s,
\end{equation}
for $t \in [0,T]$ and $n\in \N$. We make the following assumptions on the kernels $K_{\mu,n},K_{\sigma,n}$, the coefficients $\mu_n, \sigma_n$ and the initial conditions~$g_n$.

\begin{assumption}\label{ass:kernel}
  The initial conditions $g,g_n\colon[0,T]\to \R^d$ and kernels $K_\mu, K_\sigma\colon [0,T]\to \R$ and $ K_{\mu,n},K_{\sigma,n}\colon [0,T]\to \R$ for $n\in \N$ satisfy the following conditions: There are constants $\gamma\in (0,\frac{1}{2}]$, $\varepsilon>0$ and $L>0$, such that
  \begin{enumerate}
    \item[(i)] for all $n \in \bN$ the measurable functions $K_{\mu,n}, K_{\sigma,n}\colon [0,T] \to [0,\infty)$ fulfill
    \begin{align*}
      &\int_0^{T-h} |K_{\mu,n}(h+r)-K_{\mu,n}(r)|^{1+\varepsilon}\dd r + \int_0^h |K_{\mu,n}(r)|^{1+\varepsilon}\dd r \leq Lh^{\gamma(1+\varepsilon)},\\
      &\int_0^{T-h} |K_{\sigma,n}(h+r)-K_{\sigma,n}(r)|^{2+\varepsilon}\dd r  + \int_0^{h} |K_{\sigma,n}(r)|^{2+\varepsilon}\dd r  \leq Lh^{\gamma(2+\varepsilon)},
    \end{align*}
    for all $h \in [0,T]$;
    \item[(ii)] it holds that
    \begin{equation*}
      \int_0^T |K_{b,n}(s)-K_b(s)| \dd s \to 0 \text{ as } n \to \infty
    \end{equation*}
    and
    \begin{equation*}
      \int_0^T |K_{\sigma,n}(s)-K_\sigma(s)|^{2 + \varepsilon} \dd s \to 0 \text{ as } n \to \infty;
    \end{equation*}
    \item[(iii)] $g$ and $g_n$ are $\gamma$-H{\"o}lder-continuous.
   \end{enumerate}
\end{assumption}

\begin{assumption}\label{ass:coefficients2}
  Let $\mu,\mu_n\colon [0,T]\times\R^{d}\to\R^{d} $ and $\sigma,\sigma_n\colon [0,T]\times\R^{d}\to\R^{d\times m} $, $n \in \N$, be continuous functions such that:
  \begin{enumerate}
    \item[(i)] $\mu,\sigma$ and $\mu_n,\sigma_n$ are (uniformly) of linear growth, i.e. there is a constant $C_{\mu,\sigma}>0$ such that
    \begin{equation*}
      \sup_{n \in \bN} |\mu_n(t,x)|+|\sigma_n(t,x)|+ |\mu(t,x)|+|\sigma(t,x)|\leq C_{\mu,\sigma}(1+|x|),
    \end{equation*}
    for all $t\in [0,T]$ and $x\in\R^d$.
    \item[(ii)] For any compact subset $K \subset \bR^d$ we have
	      \begin{equation*}
			\lim_{n \to \infty} \sup_{t \in [0,T]} \sup_{x \in K} |\sigma_n(t,x)-\sigma(t,x)|+ |\mu_n(t,x)-\mu(t,x)|+ |g_n(t)-g(t)| = 0.
		  \end{equation*}
  \end{enumerate}
\end{assumption}

Based on the aforementioned assumptions, we obtain the following stability result for stochastic Volterra equations, generalizing the classical stability result for ordinary stochastic differential equations proven in \cite{Kaneko1988}.

\begin{theorem}\label{thm: stability of SVE}
  Suppose Assumption~\ref{ass:kernel} and Assumption~\ref{ass:coefficients2} and let $\xi\in L^p(\Omega)$ with $p > \max\{\frac{1}{\gamma},\frac{4+2\varepsilon}{\varepsilon}\}$. Moreover, suppose that there are unique $L^p$-solutions $(X_t)_{t\in[0,T]}$ and $(X^n_t)_{t\in[0,T]}$ to the SVEs \eqref{eq:SVE} and \eqref{eq:subsequence of SVEs}, for $n\in \N$, respectively. Then, one has
  \begin{equation*}
    \lim_{n \to \infty} \E \bigg[ \sup_{t \in [0,T]} |X^n_t-X_t|^2\bigg]=0.
  \end{equation*}
\end{theorem}

\begin{proof}
	See Appendix \ref{sec: appendix}.
\end{proof}

Note that the assumption on the existence of unique solutions can be ensured by postulating the coefficients to be Lipschitz continuous, see e.g. \cite{Wang2008}. However, for instance, in a one-dimensional setting a unique solution can also be obtained for H{\"o}lder continuous diffusion coefficients, see e.g. \cite{AbiJaberElEuch2019b,Promel2023}.

\medskip

Assuming that the coefficients of a SVE are Lipschitz continuous, one can quantify the stability result of Theorem~\ref{thm: stability of SVE}, as we shall present below. To that end, we consider, as comparison to the SVE~\eqref{eq:SVE}, the SVE
\begin{equation}\label{eq:SVE2}
  \tilde{X}_t = \xi\, \tilde{g}(t)+\int_0^t \tilde{K}_{\mu}(t-s)\tilde{\mu}(s,\tilde{X}_s)\dd s+\int_0^t \tilde{K}_{\sigma}(t-s)\tilde{\sigma}(s,\tilde{X}_s)\dd B_s,\quad t\in [0,T],
\end{equation}
where $\tilde{g}\colon [0,T]\to \R$ is a continuous function, and where the coefficients $\tilde{\mu}\colon [0,T]\times\R^d\to\R^d $ and $\tilde{\sigma}\colon [0,T]\times\R^d\to\R^{d\times m}$, and the convolutional kernels $\tilde{K}_\mu, \tilde{K}_\sigma\colon [0,T]\to \R$ are measurable functions.

For the convolutional kernels, we make the following assumption.

\begin{assumption}\label{ass:norms}
  Let $q,\tilde{q}> 1$ and $p\geq 2$ be such that
  \begin{equation}\label{def:p}
    \frac{1}{p}+\frac{1}{q}=1 \quad \text{and}\quad \frac{2}{p}+\frac{1}{\tilde{q}}=1.
  \end{equation}
  Suppose that $\|K_{\mu}\|_q +\|\tilde{K}_\mu\|_q<\infty$, $\|K_{\sigma}\|_{2\tilde{q}} +\|\tilde{K}_\sigma\|_{2\tilde{q}}<\infty$, and $\|g\|_\infty +\|\tilde{g}\|_\infty <\infty$.
\end{assumption}

For the coefficients, we require the standard Lipschitz and linear growth conditions.

\begin{assumption}\label{ass:coefficients}
  Let $\mu,\tilde{\mu} \colon [0,T]\times\R\to\R^d $ and $\sigma,\tilde{\sigma}\colon [0,T]\times\R\to\R^{d\times m} $ be measurable functions such that:
  \begin{enumerate}
    \item[(i)] $\mu,\sigma$ and $\tilde{\mu},\tilde{\sigma}$ are of linear growth, i.e. there is a constant $C_{\mu,\sigma}>0$ such that
    \begin{equation*}
       |\tilde{\mu}(t,x)|+|\tilde{\sigma}(t,x)|+ |\mu(t,x)|+|\sigma(t,x)|\leq C_{\mu,\sigma}(1+|x|),
    \end{equation*}
    for all $t\in [0,T]$ and $x\in\R^d$.
    \item[(ii)] $\mu,\sigma$ and $\tilde{\mu},\tilde{\sigma}$ are Lipschitz continuous in the space variable uniformly in time, i.e. there is a constant $C_{\mu,\sigma}>0$ such that
    \begin{align*}
      &|\mu(t,x)-\mu(t,y)|+|\sigma(t,x)-\sigma(t,y)|\leq C_{\mu,\sigma}|x-y| \quad \text{and}\\
      &|\tilde{\mu}(t,x)-\tilde{\mu}(t,y)|+|\tilde{\sigma}(t,x)-\tilde{\sigma}(t,y)|\leq C_{\mu,\sigma}|x-y|
    \end{align*}
    holds for all $t\in [0,T]$ and $x,y\in \R^d$.
  \end{enumerate}
\end{assumption}

Based on these assumptions, we obtain the following stability result for stochastic Volterra equations with Lipschitz continuous coefficients. For related stability results in the context ordinary stochastic differential equations, we refer to \cite[Section~3]{Kwossek2025} and the references therein.

\begin{proposition}\label{prop: stability SVE}
  Suppose Assumption~\ref{ass:norms}, Assumption~\ref{ass:coefficients} and assume $\xi\in L^p(\Omega)$. Let $(X_t)_{t\in[0,T]}$ and $(\tilde{X}_t)_{t\in[0,T]}$ be the solutions to the SVEs \eqref{eq:SVE} and \eqref{eq:SVE2}, respectively. Then, there is some constant $C>0$, depending on $\mu,\sigma,\tilde{\mu},\tilde{\sigma},{K}_\mu,{K}_\sigma,\tilde{K}_\mu,\tilde{K}_\sigma,p,\xi$, such that
  \begin{equation}\label{eq:lem_stability}
    \sup_{t\in [0,T]}\E[|X_t-\tilde{X}_t|^p]\leq C \Big( \|g-\tilde{g}\|_{\infty}^p+\|\mu-\tilde{\mu}\|_{\infty}^p+\|\sigma-\tilde{\sigma}\|_{\infty}^p+\|K_\mu-\tilde{K}_\mu\|_{q}^p + \|K_\sigma -\tilde{K}_\sigma\|_{2\tilde{q}}^p \Big).
  \end{equation}
\end{proposition}

\begin{proof}
	See Appendix \ref{sec: appendix}.
\end{proof}

\begin{remark}\label{rem: neural networks}
  The stability results, presented in Theorem~\ref{thm: stability of SVE} and Proposition~\ref{prop: stability SVE}, demonstrate the universality of neural stochastic Volterra equations like~\eqref{eq:neuralSVE}. Indeed, the unique solution of any general stochastic Volterra equations can be approximated arbitrary well by solutions of neural stochastic Volterra equations, assuming that the associated neural network converges in a suitable sense. For the coefficients the required suitable convergence is the local uniform convergence subject to a uniform global linear growth constraint. As shown in \cite{Kwossek2025}, many frequently used classes of neural networks do allow for such convergence, in particular, neural networks based on  LipSwish activation functions, as we use in the numerical experiments below. For the kernels, the suitable type of convergence is an $L^p$-convergence on the compact interval $[0,T]$. $L^p$-type universal approximation theorem can already be found in the classical work of \cite{Hornik1991}; we also refer to \cite{Kidger2020b} and the references therein.
\end{remark}

\section{Numerical experiments}\label{sec:numerical experiments}

In this section, we numerically investigate the supervised learning problem utilizing neural stochastic Volterra equations aiming to learn Volterra type dynamics such as the disturbed pendulum equation, the generalized Ornstein--Uhlenbeck process, a model for the dynamics of monetary reserves, and the rough Heston model. The performance is compared to Deep Operator Networks and neural stochastic differential equations. For all the neural SVEs, we chose the latent dimensions $d_h=d_K=12$ which experimentally proved to be well-suited. We consider the interval $[0,T]$ for $T=5$ and discretize it equally-sized using the grid size $\Delta t=0.1$.

\smallskip

As a benchmark model, we use the Deep Operator Network (DeepONet) algorithm. DeepONet is a popular class of neural learning algorithms for general operators on function spaces that was introduced in \cite{Lu2021}. A DeepONet consists of two neural networks: the branch network which operates on the function space $C([0,T];\R^n)$ (where $[0,T]$ is represented by some fixed discretization), and the so-called trunk network which operates on the evaluation point $t\in[0,T]$. Then, the output of the DeepONet is defined as
\begin{equation*}
  \text{DeepONet}(f)(t) = \sum_{k=1}^p b_kt_k+b_0,
\end{equation*}
where $(b_k)_{k=1,\dots,p}$ is the output of the branch network operating on the discretization of $f\in C([0,T];\R^n)$, $(t_k)_{k=1,\dots,p}$ is the output of the trunk network operating on $t\in[0,T]$ and $p\in\N$ is the dimension of the output of both networks. Following \cite{Lu2021}, we model both networks as feedforward networks. We perform a grid search to optimally determine the depth and width of both networks such as the activation functions, optimizer and learning rate.

\smallskip

We perform experiments on a one-dimensional disturbed pendulum equation, a one- and a two-dimensional Ornstein--Uhlenbeck equation as well as a one-dimensional rough Heston equation. We perform the experiments on low-, mid- and high-data regimes with $n=100$, $n=500$ and $n=2000$, and use $80\%$ of the data for training and $20\%$ for testing. We compare the results of neural SVEs to those of DeepONet and consider for both algorithms the mean relative $L^2$-loss. All experiments are trained for an appropriate number epochs of iterations until there is no improvement anymore. For neural SVEs, we use the Adam stochastic optimization algorithm, which heuristically proved to be well-suited, with learning rate $0.01$ and scale the learning rate by a factor $0.8$ after every $25\%$ of epochs.

Note that since DeepONet is not able to deal with random initial conditions, we use deterministic initial conditions $\xi=2$ in the DeepONet experiments. For neural SVEs we use initial conditions $\xi\sim\mathcal{N}(2,0.2)$ unless stated otherwise.

\begin{remark}
  The results in this section show that neural SVEs are able to outperform DeepONet significantly (see Table~\ref{table:1}-Table~\ref{table:rh}). Especially, neural SVEs generalize much better which can be seen in the good performance on the test sets where neural SVEs are up to $20$ times better than DeepONet. This can be explained by the explicit structure of the Volterra equation that is already part of the model for neural SVEs.
\end{remark}

All the code is published on~\url{https://github.com/davidscheffels/Neural_SVEs}.

\subsection{Disturbed pendulum equation}

As first example, we study the disturbed pendulum equation resulting from Newton's second law. Recall, general second-order differential systems (without first-order terms) perturbed by a multiplicative noise are given by
\begin{equation*}
  y^{\prime\prime}(t) = \mu(t,y(t)) + \sigma(t,y(t))\dot{B}_t, \quad t\in[0,T],
\end{equation*}
where $\dot{B}_t=\frac{d B_t}{dt}$ is White noise for some standard Brownian motion $(B_t)_{t\in[0,T]}$. Using the deterministic and the stochastic Fubini theorem, this system can be rewritten as stochastic Volterra equations
\begin{equation*}
  y(t) = y(0) + t\cdot y^\prime(0)+\int_0^t (t-u) \mu(u,y(u))\dd u +\int_0^t(t-u) \sigma(u,y(u))\dd B_u.
\end{equation*}
A concrete example from physics is the disturbed pendulum equation (see \cite[Exercise~5.12]{Oksendal2003}) resulting from Newton's second law, see e.g.~\cite[Section~2.4]{Kreyszig1999}, which describes the motion of an object $X$ with deterministic initial value $x_0$ under some force $F$, can be described by the differential equation
\begin{equation*}
  m\frac{\dd^2 X(t)}{\dd t^2}=F(X(t)),\quad X(0)=x_0.
\end{equation*}
Hence, $(X_t)_{t\in[0,T]}$ solves the SVE
\begin{equation*}
  X(t)= x_0+tX^\prime(0)+\int_0^t(t-s)\frac{F(X(s))}{m}\dd s + \int_0^t(t-s)\frac{\varepsilon X_s}{m}\dd B_s.
\end{equation*}
As prototyping example, we consider the one-dimensional equation
\begin{equation}\label{eq:exp_disturbedpendulum}
  y_t = \xi - \int_0^t (t-s)y_s\dd s +\int_0^t (t-s)y_s\dd B_s,\qquad t\in[0,T],
\end{equation}
with the target to learn its dynamics by neural SVEs and DeepONet. The results are presented in Table~\ref{table:1}.

\begin{table}[H]
  \begin{tabular}{ |c|c|c|c| }
  \hline
  \textbf{Neural SVE} & Train set & Test set \\
  \hline
  $n=100$ & $0.01$ & $0.013$ \\
  $n=500$& $0.008$ & $0.008$ \\
  $n=2000$& $0.006$ & $0.006$ \\
  \hline
  \end{tabular}
  \hspace*{2pt}
  \begin{tabular}{ |c|c|c|c| }
  \hline
  \textbf{DeepONet} & Train set & Test set \\
  \hline
  $n=100$ & $0.003$ & $0.2$ \\
  $n=500$& $0.003$ & $0.06$ \\
  $n=2000$& $0.003$ & $0.02$ \\
  \hline
  \end{tabular}
  \caption{Mean relative $L^2$-losses after training for the disturbed pendulum equation~\eqref{eq:exp_disturbedpendulum}.}
  \label{table:1}
\end{table}

Sample paths of the training and the testing sets together with their learned approximations are shown in Figure~\ref{figure:paths_pen}. It is clearly visible that while DeepONet is not able to generalize properly to the testing set, the learned neural SVE paths are very close to the true paths also for the test set.

\begin{figure}[H]
  \begin{tblr}{|m{3cm}|m{3cm}|m{3cm}|m{3cm}|}
  \hline
  Neural SVE: \newline Training set & Neural SVE: \newline Test set & DeepONet: \newline Training set & DeepONet: \newline Test set\\
  \hline
  \includegraphics[width=2.8cm, height=5cm]{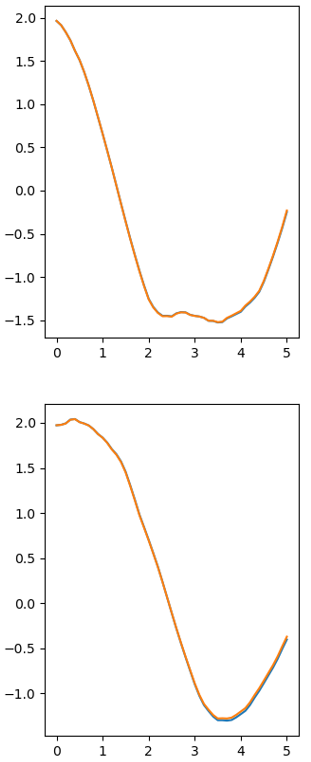} & \includegraphics[width=2.8cm, height=5cm]{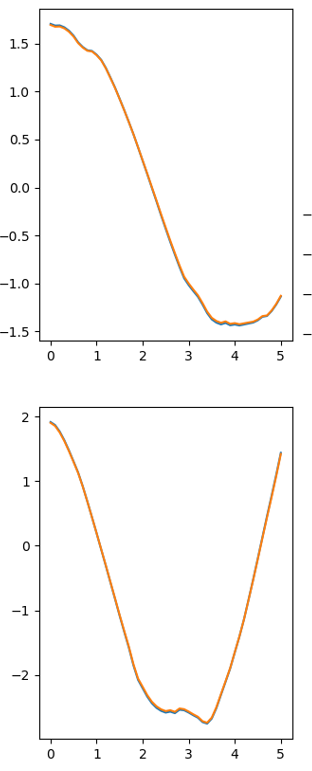} & \includegraphics[width=2.8cm, height=5cm]{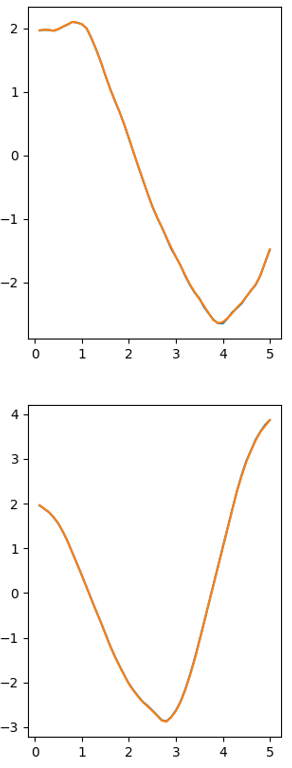} & \includegraphics[width=2.8cm, height=5cm]{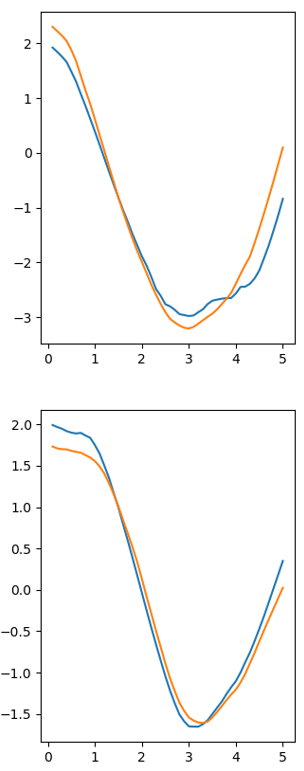}\\
  \hline
  \end{tblr}
  \caption{Sample neural SVE and DeepONet paths from the training and the test set for the disturbed pendulum equation and $n=100$. Blue (barely visible) are the original paths and orange the learned approximations.}
  \label{figure:paths_pen}
\end{figure}

To highlight the performance of Neural SVE in a more complex setting, we take a nonlinear coefficient~$\mu$ and consider
\begin{equation}\label{eq:NPEN}
  y_t = \xi - \int_0^t (t-s) \sin(y_s)\dd s +\int_0^t 0.4(t-s)y_s\dd B_s,\qquad t\in[0,T].
\end{equation}
The results are presented in Table \ref{table:NPEN}. It stands out immediately that for $n=100$ the performance in the test set is far worse than in the training set. Most solutions remain in the range of $[-\pi,\pi]$. Yet, in some rare cases, solutions may explode outside this range. An example of this can be seen in Figure \ref{figure:paths_NPEN}. In the training set this happens too rarely for the neural network to learn the functions outside of $[-\pi,\pi]$ precisely. Table \ref{table:NPEN} clearly shows that this discrepancy disappears when the training set is sufficiently large. In the case of $n = 2000$ one also sees that the neural network can learn these complex functions as good as in the prior example. Then the plots for the test set look similar to the training set plots of the Neural SVE trained with $n=100$ trajectories.

\begin{table}[H]
  \begin{tabular}{ |c|c|c|c| }
  \hline
  \textbf{Neural SVE} & Train set & Test set \\
  \hline
  $n=100$ & $0.007$ & $0.028$ \\
  $n=500$& $0.012$ & $0.014$ \\
  $n=2000$& $0.007$ & $0.006$ \\
  \hline
  \end{tabular}
  \hspace*{2pt}
  \begin{tabular}{ |c|c|c|c| }
  \hline
  \textbf{DeepONet} & Train set & Test set \\
  \hline
  $n=100$ & $0.004$ & $0.23$ \\
  $n=500$& $0.006$ & $0.13$ \\
  $n=2000$& $0.004$ & $0.06$ \\
  \hline
  \end{tabular}
  \caption{Mean relative $L^2$-losses after training for the nonlinear disturbed pendulum equation~\eqref{eq:NPEN}.}
  \label{table:NPEN}
\end{table}

\begin{figure}[H]
  \begin{tblr}{|m{3cm}|m{3cm}|m{3cm}|m{3cm}|}
  \hline
  Neural SVE: \newline Training set & Neural SVE: \newline Test set & DeepONet: \newline Training set & DeepONet: \newline Test set\\
  \hline
  \includegraphics[width=2.8cm, height=5cm]{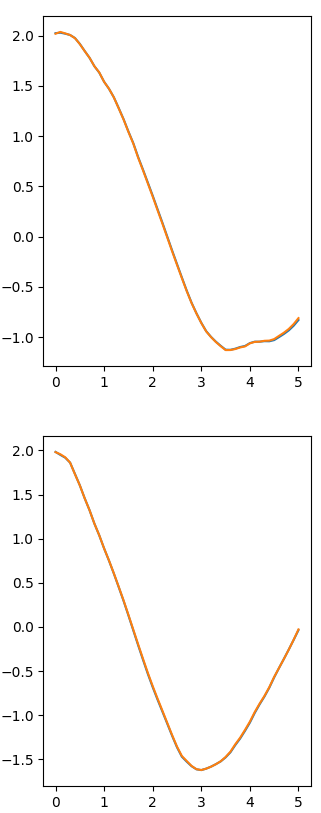} & \includegraphics[width=2.8cm, height=5cm]{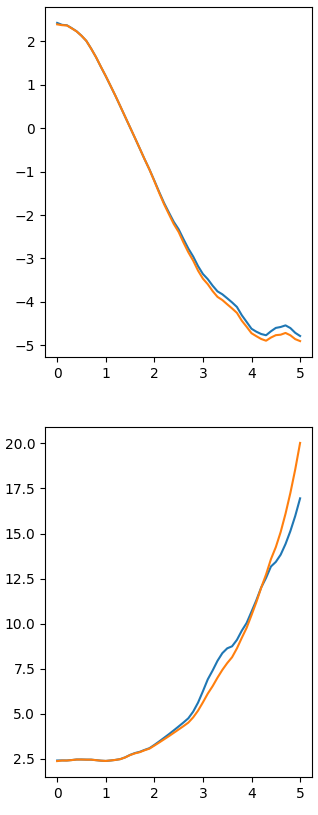} & \includegraphics[width=2.8cm, height=5cm]{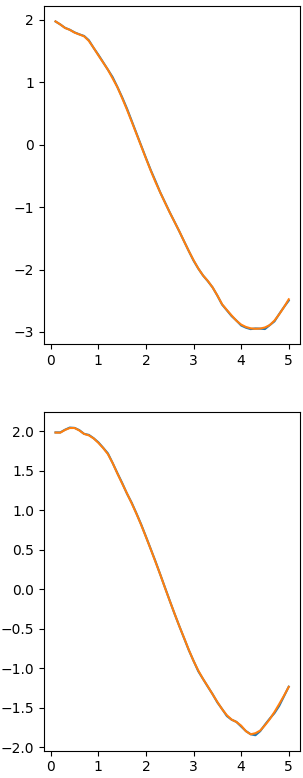} & \includegraphics[width=2.8cm, height=5cm]{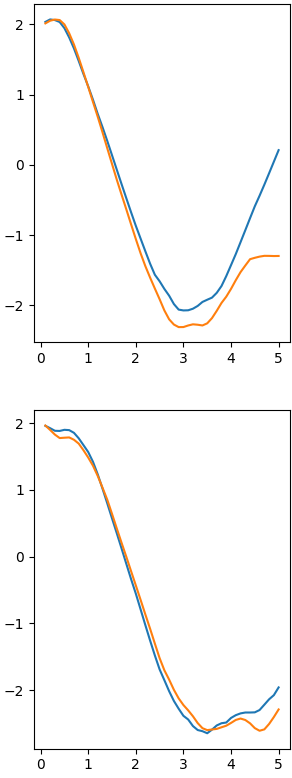}\\
  \hline
  \end{tblr}
  \caption{Sample neural SVE and DeepONet paths from the training and the test set for the disturbed pendulum equation with nonlinear drift and $n=100$. Blue are the original paths and orange the learned approximations.}
  \label{figure:paths_NPEN}
\end{figure}

\subsection{Rough Heston equation}

The rough Heston model is one of the most prominent representatives of rough volatility models in mathematical finance, see e.g. \cite{ElEuch2019,AbiJaberElEuch2019b}, where the volatility process is modeled by SVEs with the singular kernels $(t-s)^{-\alpha}$ for some $\alpha\in(0,1/2)$, that is
\begin{equation*}
  V_t = V_0 +\frac{1}{\Gamma(\alpha)}\int_0^t (t-s)^{-\alpha}\lambda(\theta-V_s)\dd s+\frac{\lambda\nu}{\Gamma(\alpha)}\int_0^t (t-s)^{-\alpha}\sqrt{|V_s|}\dd B_s,\quad t\in[0,T],
\end{equation*}
where $\Gamma(x)=\int_0^\infty t^{x-1}e^{-t}\dd t$ denotes the real valued Gamma function, and $\lambda,\theta,\nu\in\R$. As specific example, we consider the one-dimensional equation
\begin{equation}\label{eq:exp_rh}
  V_t = \xi + \frac{1}{\Gamma(0.4)}\int_0^t (t-s)^{-0.4}(2-V_s)\dd s +\frac{1}{\Gamma(0.4)}\int_0^t (t-s)^{-0.4}\sqrt{|V_s|}\dd B_s,\qquad t\in[0,T],
\end{equation}
with the target to learn its dynamics by neural SVEs and DeepONet. The results are presented in Table~\ref{table:rh}. Neural SVEs outperform DeepONet here by far.

\begin{table}[H]
  \begin{tabular}{ |c|c|c|c| }
  \hline
  \textbf{Neural SVE} & Train set & Test set \\
  \hline
  $n=100$ & $0.003$ & $0.003$ \\
  $n=500$& $0.0025$ & $0.0028$ \\
  $n=2000$& $0.0015$ & $0.0017$ \\
  \hline
  \end{tabular}
  \hspace*{2pt}
  \begin{tabular}{ |c|c|c|c| }
  \hline
  \textbf{DeepONet} & Train set & Test set \\
  \hline
  $n=100$ & $0.035$ & $0.13$ \\
  $n=500$& $0.004$ & $0.037$ \\
  $n=2000$& $0.003$ & $0.014$ \\
  \hline
  \end{tabular}
  \caption{Mean relative $L^2$-losses after training for the rough Heston equation~\eqref{eq:exp_rh}.}
  \label{table:rh}
\end{table}

Sample paths of the training and the testing sets together with their learned approximations are shown in Figure~\ref{figure:paths_rh}.

\begin{figure}[H]
  \begin{tblr}{|m{3cm}|m{3cm}|m{3cm}|m{3cm}|}
  \hline
  Neural SVE: \newline Training set & Neural SVE: \newline Test set & DeepONet: \newline Training set & DeepONet: \newline Test set\\
  \hline
  \includegraphics[width=2.8cm, height=5cm]{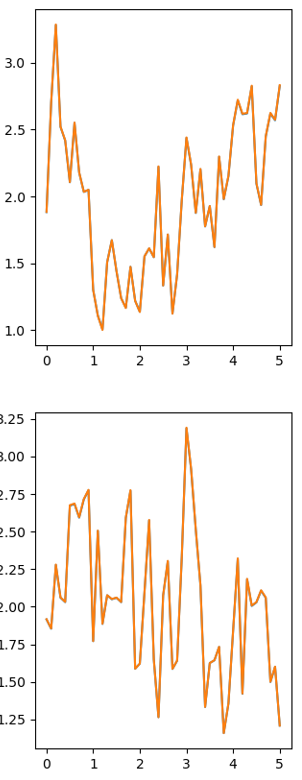} & \includegraphics[width=2.8cm, height=5cm]{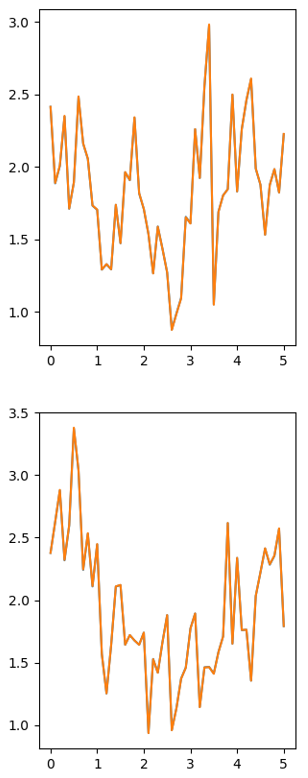} & \includegraphics[width=2.8cm, height=5cm]{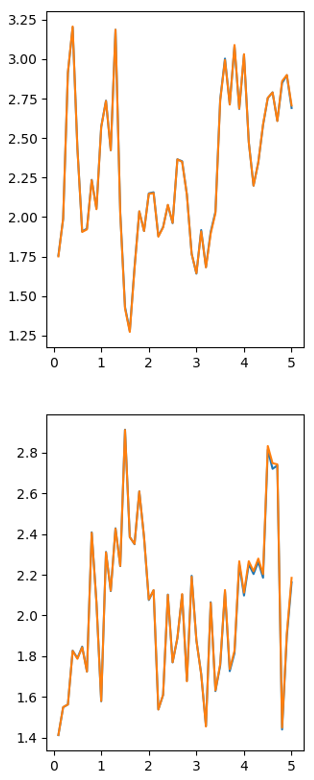} & \includegraphics[width=2.8cm, height=5cm]{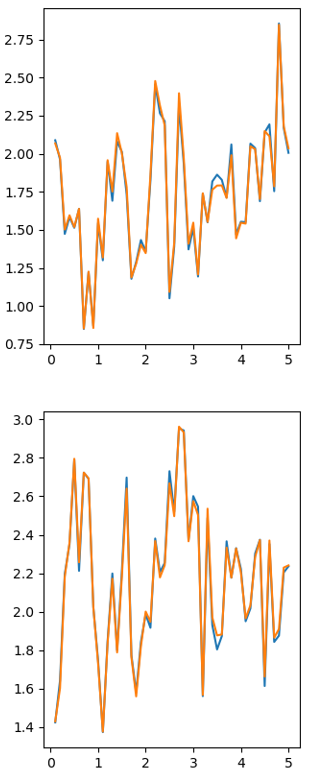}\\
  \hline
  \end{tblr}
  \caption{Sample neural SVE and DeepONet paths from the training and the test set for the rough Heston equation and $n=2000$. Blue (barely visible) are the original paths and orange the learned approximations.}
  \label{figure:paths_rh}
\end{figure}

\subsection{Generalized Ornstein--Uhlenbeck process}

The Ornstein--Uhlenbeck process, introduced in \cite{Uhlenbeck1930}, is a commonly used stochastic process with applications in finance, physics or biology, see e.g. \cite{Vasicek2012,Tuzel1999,Martins1994}. We consider the generalized Ornstein--Uhlenbeck process that is given by the stochastic differential equation
\begin{equation*}
  \dd X_t = \theta (\mu(t,X_t)-X_t)\dd t + \sigma(t,X_t)\dd B_t, \quad t\in[0,T],
\end{equation*}
which, using It{\^o}'s formula, can be equivalently rewritten as the SVE
\begin{equation*}
  X_t = X_0 e^{-\theta t} + \theta\int_0^t e^{-\theta (t-s)}\mu(s,X_s)\dd s + \int_0^t e^{-\theta (t-s)}\sigma(s,X_s)\dd B_s,\quad t\in[0,T].
\end{equation*}
As prototyping example, we consider the one-dimensional equation
\begin{equation}\label{eq:exp_ou}
  X_t= \xi e^{-t} + \int_0^t e^{-(t-s)}X_s\dd s +\int_0^t e^{-(t-s)}\sqrt{|X_s|}\dd B_s,\qquad t\in[0,T],
\end{equation}
with the target to learn its dynamics by neural SVEs and DeepONet. The results are presented in Table~\ref{table:ou}.

\begin{table}[H]
  \begin{tabular}{ |c|c|c|c| }
  \hline
  \textbf{Neural SVE} & Train set & Test set \\
  \hline 
  $n=100$ & $0.015$ & $0.038$ \\
  $n=500$& $0.014$ & $0.036$ \\
  $n=2000$& $0.014$ & $0.02$ \\
  \hline
  \end{tabular}
  \hspace*{2pt}
  \begin{tabular}{ |c|c|c|c| }
  \hline
  \textbf{DeepONet} & Train set & Test set \\
  \hline
  $n=100$ & $0.025$ & $0.23$ \\
  $n=500$& $0.018$ & $0.15$ \\
  $n=2000$& $0.028$ & $0.12$ \\
  \hline
  \end{tabular}
  \caption{Mean relative $L^2$-losses after training for the one-dimensional Ornstein--Uhlenbeck equation~\eqref{eq:exp_ou}.}
  \label{table:ou}
\end{table}

Sample paths of the training and the testing sets together with their learned approximations are shown in Figure~\ref{figure:paths_ou}.

\begin{figure}[H]
  \begin{tblr}{|m{3cm}|m{3cm}|m{3cm}|m{3cm}|}
  \hline
  Neural SVE: \newline Training set & Neural SVE: \newline Test set & DeepONet: \newline Training set & DeepONet: \newline Test set\\
  \hline
  \includegraphics[width=2.8cm, height=5cm]{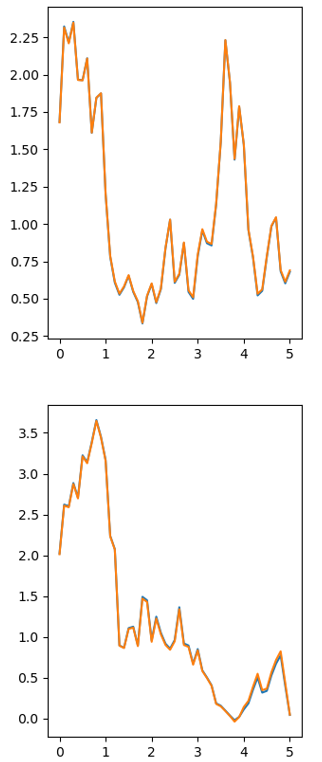} & \includegraphics[width=2.8cm, height=5cm]{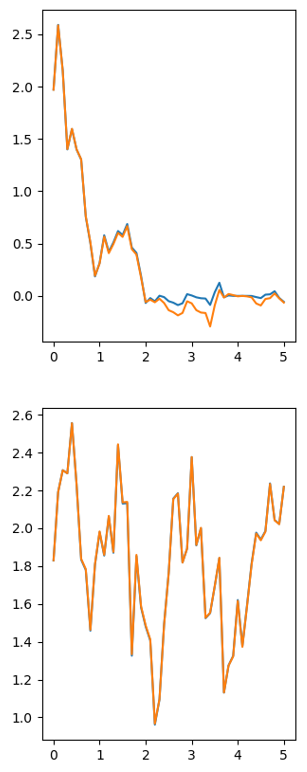} & \includegraphics[width=2.8cm, height=5cm]{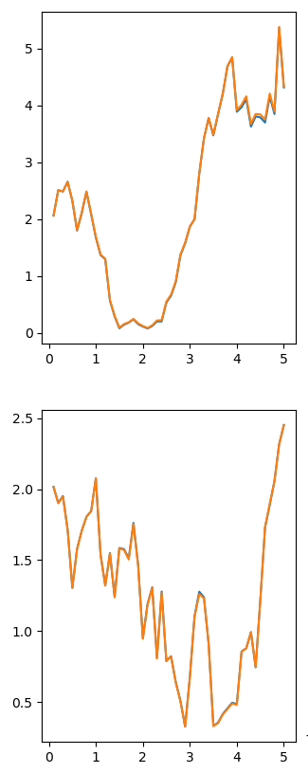} & \includegraphics[width=2.8cm, height=5cm]{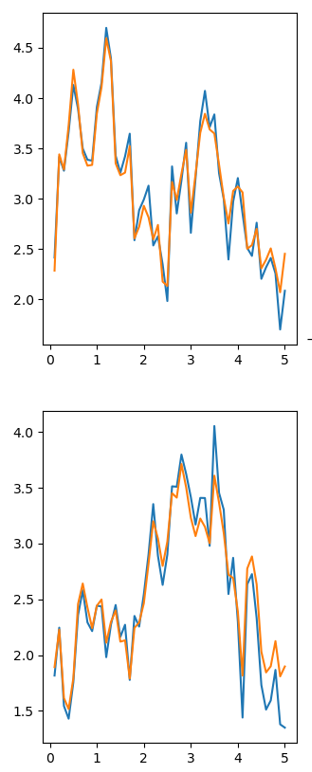}\\
  \hline
  \end{tblr}
  \caption{Sample neural SVE and DeepONet paths from the training and the test set for the one-dimensional Ornstein--Uhlenbeck equation and $n=500$. Blue (barely visible) are the original paths and orange the learned approximations.}
  \label{figure:paths_ou}
\end{figure}

Moreover, neural SVEs are able to learn multi-dimensional SVEs. As an example, we consider the two-dimensional equation 
\begin{equation}
  \begin{pmatrix}
  X^1_t\\X_t^2
  \end{pmatrix} =
  \begin{pmatrix}
  \xi_1\\
  \xi_2
  \end{pmatrix}e^{-t} + \int_0^t e^{-(t-s)}\begin{pmatrix}
  X^1_s\\X_s^2
  \end{pmatrix}\dd s
  +\int_0^t e^{-(t-s)}\begin{pmatrix}
  \sqrt{|X^1_s|},0\\0,\sqrt{|X_s^2|}
  \end{pmatrix}\dd B_s,\qquad t\in[0,T],\label{eq:exp_ou_2d}
\end{equation}
where $B$ is a $2$-dimensional Brownian motion, and try to learn its dynamics by neural SVEs. The results are presented in Table~\ref{table:ou_2d}.

\begin{table}[H]
  \begin{tabular}{ |c|c|c|c| }
  \hline
  \textbf{Neural SVE} & Train set & Test set \\
  \hline
  $n=100$ & $0.038$ & $0.095$ \\
  $n=500$& $0.04$ & $0.085$ \\
  $n=2000$& $0.038$ & $0.04$ \\
  \hline
  \end{tabular}
  \caption{Mean relative $L^2$-losses after training for the two-dimensional Ornstein--Uhlenbeck equation \eqref{eq:exp_ou_2d}.}
  \label{table:ou_2d}
\end{table}

Sample paths of the training and the testing sets together with their learned approximations are shown in Figure~\ref{figure:paths_ou_2d}.

\begin{figure}[H]
  \begin{tblr}{|m{6cm}|m{6cm}|}
  \hline
  Neural SVE: Training set & Neural SVE: Test set \\
  \hline
  \includegraphics[width=6cm, height=5cm]{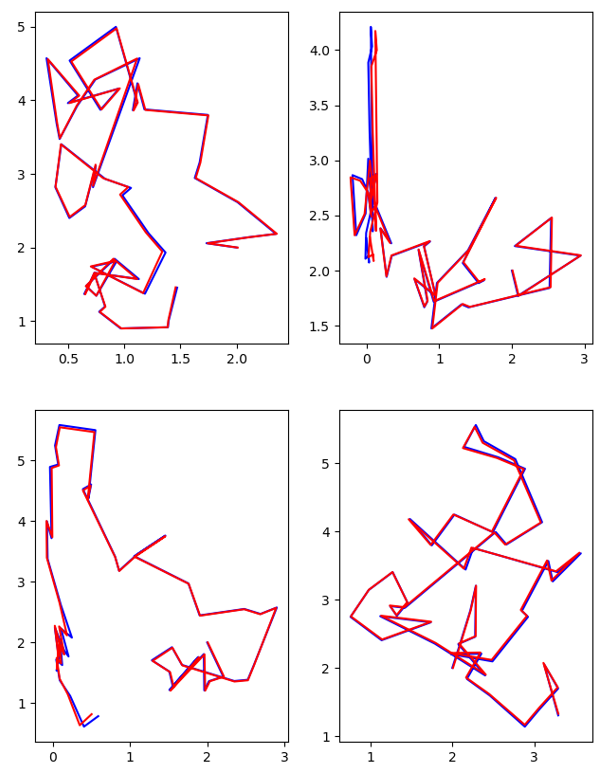} & \includegraphics[width=6cm, height=5cm]{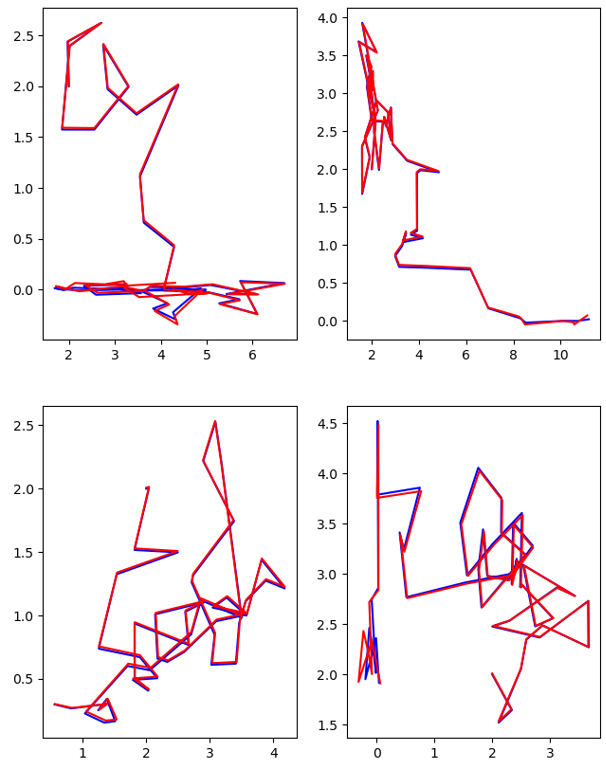} \\
  \hline
  \end{tblr}
  \caption{Sample neural SVE paths from the training and the test set for the two-dimensional Ornstein--Uhlenbeck equation and $n=2000$. Blue (barely visible) are the original paths and red the learned approximations.}
  \label{figure:paths_ou_2d}
\end{figure}

\subsection{Monetary reserve modelling}

For a higher-dimensional model with cross-dependency between the different trajectories we simulate and approximate the bank run model as in \cite{Carmona2018c}. There, the dynamics of the log-monetary reserves of $N$ banks are modelled by the coupled diffusion processes $X^i$, $i = 1 , \ldots, N$,
\begin{equation}\label{eq:br}
	\d X_t^i = (\alpha^i_t - \alpha_{t - \tau}^i ) \dd t + \sigma \dd W_t^i, \qquad 0 \leq t \leq T,
\end{equation}
where $W^i$, $i = 1, \ldots, N$, are independent standard Brownian motioins, and the rate of borrowing or lending $\alpha_t^i$ represents the control of the bank $i$ on the system. Vice versa, the delayed control $\alpha_{t-\tau}^i$ represents the repayments after a fixed time $\tau$. In the paper mentioned the authors solve the differential game where bank $i$, $i = 1, \ldots,  N$, aims to minimize its objective function
\begin{equation*}
	J^i(\alpha) = \bE\Big[\int_0^T f_i(X_t,\alpha_t^i) \dd t + g_i(X_T)\Big]
\end{equation*}
with $f_i(x, \alpha^i)$ and $g_i(x,\alpha^i)$ heavily depending on $\frac{1}{N} \sum_{j=1}^N x_j - x_i$. The optimal control $\alpha^\ast$ can only be stated in terms of multiple differential equations with no closed-form solution, therefore we choose
\begin{equation*}
	\alpha_t^i = \Big(0.1 + 0.5 \sin\Big(\frac{\pi t}{2T}\Big)\Big) (\bar{X}_t - X_t^i).
\end{equation*}
Here, we choose $X_0^i \sim\mathcal{N}(10,1)$, $\sigma= 0.05$, $T=50$, $\Delta t = 1$ and $\tau = 10$, opening up to the interpretation of each time unit corresponding to one day and borrowing (lending) decisions being made on a daily basis. Note that this problem is highly complex since $\bar{X}$ is not an input to the neural network, so the network has to learn how $X^i$ depends on $X^j$, $j=1, \ldots, N$, itself. To account for the higher complexity we opted for a network with latent dimension $d_h=d_K=24$ and used $n=1000$ datasets with $10$ banks each.

As shown in Table~\ref{table:br}, the neural SVE learned the dynamics quite well also in this example.

\begin{table}[H]
  \begin{tabular}{ |c|c|c|c| }
  \hline
  \textbf{Neural SVE} & Train set & Test set\\
  \hline
  $n=1000$ & $0.037$ & $0.039$ \\
  \hline
  \end{tabular}
  \caption{Mean relative $L^2$-losses after training for the monetary reserve model~\eqref{eq:br}.}
  \label{table:br}
\end{table}

\begin{figure}[H]
  \begin{tblr}{|m{14cm}|}
  \hline
  Training set\\
  \hline
  \includegraphics[width=14cm]{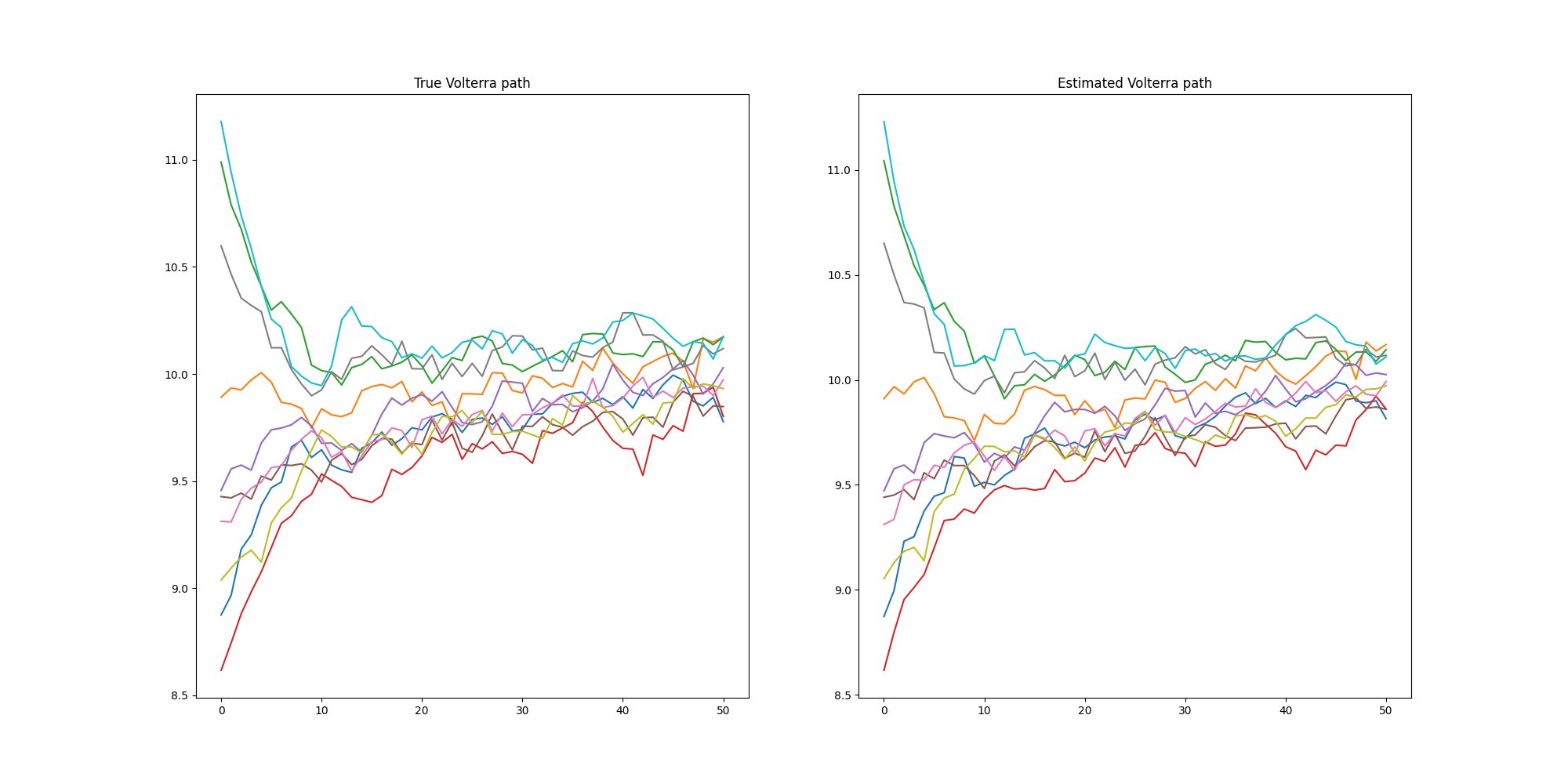}\\
  \hline
  Test set\\
  \hline
  \includegraphics[width=14cm]{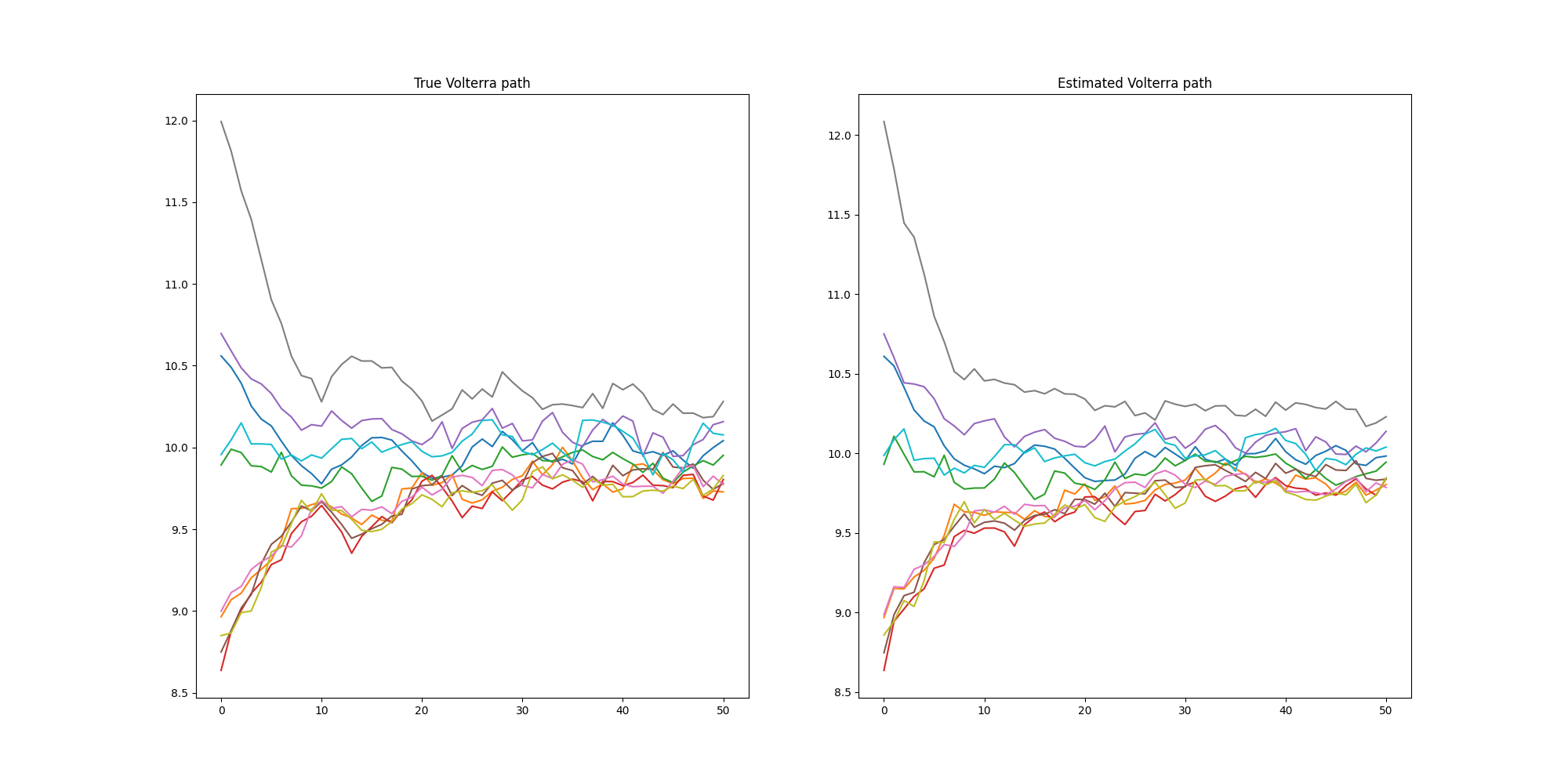} \\
  \hline
  \end{tblr}
  \caption{Sample paths from the training and test sets for the monetary reserve model. On the left side one can see the paths as from the data set, on the right side the learnt approximation by the SVE. The same color corresponds to the same bank.}
  \label{figure:paths_br}
\end{figure}

In Figure~\ref{figure:paths_br} one can see the log-monetary reserves of the $10$ banks from one data set and its estimation from the test and the training set each. The dips when the first borrowing contracts end after $t=10$ are not as pronounced in the network's approximation as in the original data set. This likely stems from neural networks struggling to learn indicator functions such as $K_\mu$ here. Still, the order of the banks' reserves as well as their magnitude are captured very well.

\subsection{Comparison to neural SDEs}\label{subsec: neural SDEs}

Introduced in \cite{Kidger2022}, a \textit{neural stochastic differential equation} (neural SDE) is defined by
\begin{align}
  Z_0 &= L_\theta(\xi),\notag\\
  Z_t &= Z_0\, g_\theta(t)+\int_0^t \mu_\theta(s,Z_s)\dd s + \int_0^t \sigma_\theta(s,Z_s)\dd B_s,\notag\\
  X_t&=\Pi_\theta(Z_t),\quad t\in [0,T],\notag
\end{align}
where all objects are defined as in the neural SVE~\eqref{eq:neuralSVE}. Since a neural SDE does not possess the kernel functions $K_{\mu,\theta}$ and $K_{\sigma,\theta}$ compared to the neural SVE~\eqref{eq:neuralSVE}, it is not able to fully capture the dynamics induced by SVEs.

\smallskip

Note that due to the need of discretizing the time interval when it comes to computations,  some of the properties introduced by the kernels are attenuated. However, the memory structure of an SVE is a property which can be learned by a neural SVE but, in general, not by a neural SDE since SDEs posses the Markov property. Therefore, to see the potential capabilities of neural SVEs compared to neural SDEs, it is best to look at examples where the dependency on the whole path plays a crucial role. To construct such an example, we consider the kernels
\begin{equation*}
  K_\mu(s,t):=K_\sigma(s,t):=K(t-s)=\begin{cases}1,\quad &\text{if }(t-s)\leq T/4,\\
  -1,\quad &\text{if }(t-s)>T/4,
  \end{cases}
\end{equation*}
and aim to learn the dynamics to the one-dimensional SVE
\begin{equation}\label{eq:RH_longrun}
  X_t = \xi + \int_0^t K(t-s)(2-X_s)\dd s +\int_0^t K(t-s)\sqrt{|X_s|}\dd B_s,\qquad t\in[0,T],
\end{equation}
where $\xi\sim \mathcal{N}(5,0.5)$ and $T=5$. The process $(X_t)_{t\in[0,5]}$ is expected to decrease in the first quarter of the interval $[0,5]$ where $K(t-s)=1$ holds due to the mean-reverting effect of the drift coefficient $\mu(s,x)=2-x$, then something unpredictable will happen and finally in the last part of the interval $t\in [0,5]$ where the kernels attain $-1$ for a large proportion of $s\in[0,t]$, the process might become big due to the turning sign in the drift. Hence, it is expected that the path dependency will have a substantial impact.

We learn the dynamics of equation~\eqref{eq:RH_longrun} simulated on an equally-sized grid with grid size $\Delta t=0.1$  by a neural SDE and by a neural SVE for a dataset of size $n=500$ and compare the results in Table~\ref{table:rh_longrun}. It can be observed that the neural SDE fails to learn the dynamics of \eqref{eq:RH_longrun} properly while the neural SVE performs well.

\begin{table}[H]
  \begin{tabular}{ |c|c|c|c| }
  \hline
  \textbf{Neural SVE} & Train set & Test set \\
  \hline
  $n=500$& $0.008$ & $0.009$ \\
  \hline
  \end{tabular}
  \hspace*{2pt}
  \begin{tabular}{ |c|c|c|c| }
  \hline
  \textbf{Neural SDE} & Train set & Test set \\
  \hline
  $n=500$& $0.19$ & $0.21$ \\
  \hline
  \end{tabular}
  \caption{Mean relative $L^2$-losses after training for the SVE~\eqref{eq:RH_longrun}.}
  \label{table:rh_longrun}
\end{table}

Sample paths of the training and the testing sets together with their learned approximations are shown in Figure~\ref{figure:paths_rh_longrun}.

\begin{figure}[H]
  \begin{tblr}{|m{3cm}|m{3cm}|m{3cm}|m{3cm}|}
  \hline
  Neural SVE: \newline Training set & Neural SVE: \newline Test set & Neural SDE: \newline Training set & Neural SDE: \newline Test set\\
  \hline
  \includegraphics[width=3.15cm, height=8cm]{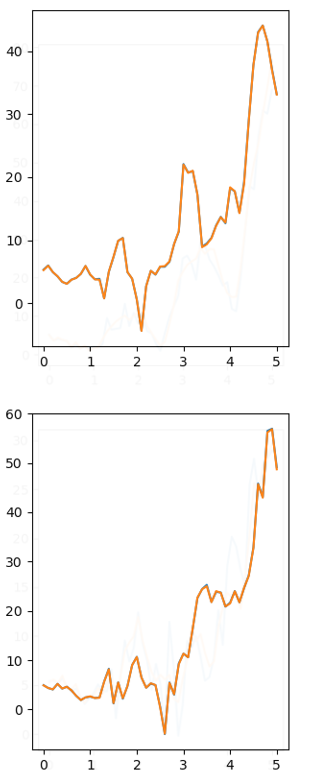} & \includegraphics[width=3.15cm, height=8cm]{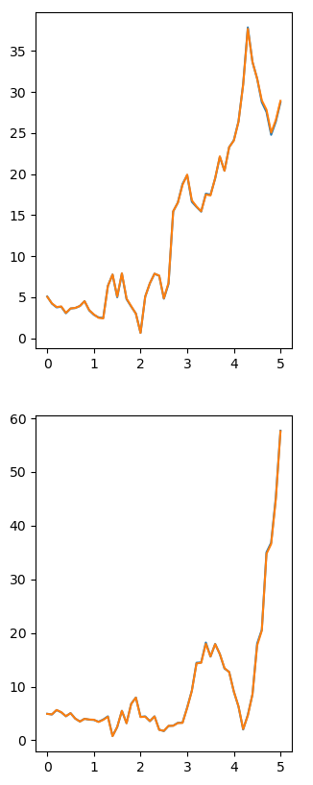} & \includegraphics[width=3.15cm, height=8cm]{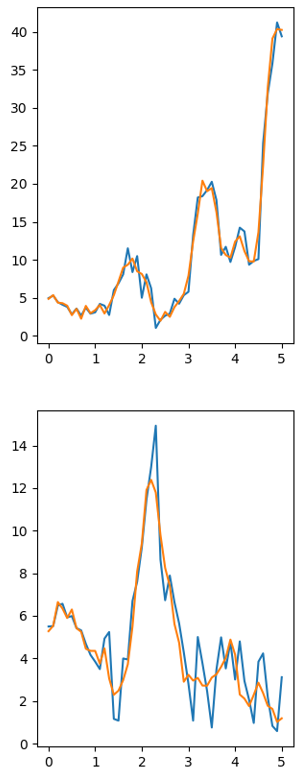} & \includegraphics[width=3.15cm, height=8cm]{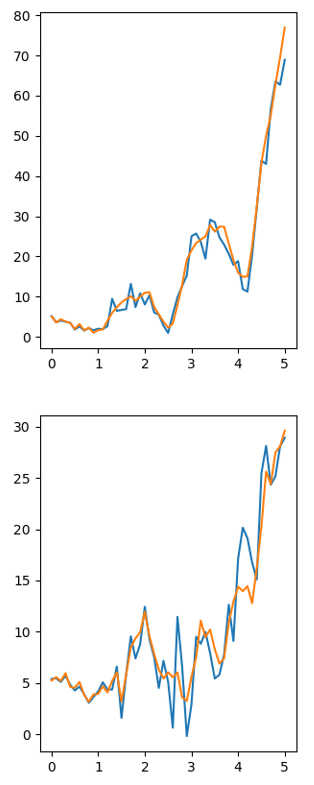}\\
  \hline
  \end{tblr}
  \caption{Sample neural SVE and neural SDE paths from the training and the test set for the SVE~\eqref{eq:RH_longrun} and $n=500$. Blue are the original paths and orange the learned approximations.}
  \label{figure:paths_rh_longrun}
\end{figure}

\subsection{Computational aspects}

Next we briefly analyse computational aspects of the Neural SVE, most particularly, its runtime and memory usage. For each property we outline the influence of the number of epochs, the grid size $\Delta t$, the terminal time $T$, the dimension of the SVE, the latent dimension and the sample-size~$n$.

\medskip

All computations were made using an AMD Ryzen 7 5800X processor. Using CUDA with the NVIDIA GeForce RTX 3060 roughly triples the runtime. This is likely due to the relatively small size, especially the small width, of the neural network. Note that for a one-dimensional SVE there are 1264 parameters to be trained, for a two-dimensional SVE there are 1901 parameters. These parameters are basically split into five different neural networks (one for each kernel and coefficient as well as one for $g$). Also solving an SVE cannot be parallelized due to the past-dependency of the solution. 

As base case let us take the parameters as in the beginning of this section with $n=500$, a batch size of $50$ and $1000$ epochs. Runtime and memory usage scale linearly in the sample size. During the training, the runtime per iteration remains roughly constant. Hence, the runtime grows linearly in the number of epochs. The memory usage is basically independent of the number of epochs. We therefore focus on the runtime per epoch. The results are reported in Table \ref{table:comp_perf}. Doubling the latent dimension from 12 to 24 increases the number of parameters in the one-dimensional model to 4540. The impact to the runtime is relatively low, it increases to barely 17 seconds. A two-dimensional model with latent dimension 24 has 6965 parameters. Only when rapidly increasing the latent dimension one can see a significant effect. The very small dependency of runtime and memory on the number of parameters to be learned indicates that the main effort lies in solving the SVE numerically. In order to improve performance of Neural SVE one should focus on finding a more efficient numerical scheme to solve SVEs in the first place.

\begin{table}[H]
	\begin{tabular}{|l|crr|}
	\hline
	\multirow{2}{*}{} & \multicolumn{3}{c|}{Latent dimension = 12}                                                                 \\ \cline{2-4} 
	                  & \multicolumn{1}{c|}{Parameters} & \multicolumn{1}{c|}{Training time per it.} & \multicolumn{1}{c|}{Memory} \\ \hline
	dim = 1           & \multicolumn{1}{r|}{1264}       & \multicolumn{1}{r|}{16,04}                 & 1233                        \\ \hline
	dim = 2           & \multicolumn{1}{r|}{1901}       & \multicolumn{1}{r|}{15,91}                 & 1232                        \\ \hline
	\multirow{2}{*}{} & \multicolumn{3}{c|}{Latent dimension = 24}                                                                 \\ \cline{2-4} 
	                  & \multicolumn{1}{c|}{Parameters} & \multicolumn{1}{c|}{Training time per it.} & \multicolumn{1}{c|}{Memory} \\ \hline
	dim = 1           & \multicolumn{1}{r|}{4540}       & \multicolumn{1}{r|}{16,08}                 & 1231                        \\ \hline
	dim = 2           & \multicolumn{1}{r|}{6965}       & \multicolumn{1}{r|}{16,35}                 & 1233                        \\ \hline
	\multirow{2}{*}{} & \multicolumn{3}{c|}{Latent dimension = 120}                                                                \\ \cline{2-4} 
	                  & \multicolumn{1}{c|}{Parameters} & \multicolumn{1}{c|}{Training time per it.} & \multicolumn{1}{c|}{Memory} \\ \hline
	dim = 1           & \multicolumn{1}{r|}{103324}     & \multicolumn{1}{r|}{18,70}                 & 1234                        \\ \hline
	dim = 2           & \multicolumn{1}{r|}{161525}     & \multicolumn{1}{r|}{20,09}                 & 1263                        \\ \hline
	\end{tabular}
	\caption{Computational performance of neural networks trained to model stochastic Volterra equations. Reported are the number of parameters, average training time per iteration (in seconds), and peak memory usage (in MB) for varying input dimensions (dim) and latent dimensions of the network.}
  \label{table:comp_perf}
\end{table}

Changing the grid size $\Delta t$ has an influence that is counterintuitive at first glance: Halving the grid size quadruples the runtime. Obviously, for a fixed terminal time $T$, one has to evaluate at (nearly) twice as many evaluation points. But the numerical scheme for solving an SVE requires the whole past, that now also contains twice as many points to consider. Doubling the terminal time $T$ while keeping the grid size $\Delta t$ fixed has the same impact.

\begin{remark} 
  Our procedure corresponds to a discretise-then-optimize approach in classical Neural SDE. For Neural SDEs one may also consider an optimize-then-discretize approach. The optimize-then-discretize approach requires less memory but it is slower and and may lead to an inaccurate solution \cite{Kidger2022}. Since the optimize-then-discretize approach transforms the SDE into an backward SDE for the backpropagation, it is infeasible for Neural SVE due to their non-markovian structure.
\end{remark}

\appendix
\section{Proofs of Theorem~\ref{thm: stability of SVE} and Proposition~\ref{prop: stability SVE}}\label{sec: appendix}

In this appendix, we present the proofs Theorem~\ref{thm: stability of SVE} and of Proposition~\ref{prop: stability SVE}.

\begin{proof}[Proof of Theorem~\ref{thm: stability of SVE}]
  We provide a proof by contradiction. To that end, we assume that there are $\delta > 0$ and an increasing sequence $(n_k)_{k \in \bN} \subset \bN$ satisfying
  \begin{equation*}
    \inf_{k \in \bN} \bE\bigg[ \sup_{t \in [0,T]} |X^{n_k}_t-X_t|^2\bigg] \geq \delta.
  \end{equation*}
  Moreover, we define
  \begin{equation*}
    A^n_t := \int_0^t \mu_n(s,X^{n}_s)\dd s \quad\text{and}\quad
    M^n_t := \int_0^t \sigma_n(s,X^n_s) \dd s
  \end{equation*}
  for $t \in [0,T]$, $n \in \bN$.

  As in the proof of \cite[Lemma~3.8]{Promel2022}, we can obtain the tightness of probability measure
  \begin{equation*}
    \bP_{X, X^{n_k}, A^{n_k}, M^{n_k}, B}, \quad k \in \N,
  \end{equation*}
  which denotes the probability distribution of the corresponding random vector
  \begin{equation*}
    (X, X^{n_k}, A^{n_k}, M^{n_k}, B).
  \end{equation*}
  Using Prokhorov's theorem and the Skorokhod representation theorem, one deduces that there is a probability space $(\hat{\Omega}, \hat{\mathcal{F}}, \hat{\bP})$ with continuous stochastic processes $\hat{X}^l, \hat{Y}^l, \hat{B}^l, \hat{A}^l, \hat{M}^l$, $l \in \bN$ and $\hat{X}, \hat{Y}, \hat{B}, \hat{A}, \hat{M}$ such that
  \begin{equation*}
    \big(\hat{\xi}^l,\hat{X}^l, \hat{Y}^l, \hat{B}^l, \hat{A}^l, \hat{M}^l\big) \stackrel{\mathscr{D}}{\sim} \big(\xi,X, X^{n_{k_l}},B,A^{n_{k_l}}, M^{n_{k_l}}\big), \qquad l \in \bN,
  \end{equation*}
  and
  \begin{equation*}
   (\hat{X}^l, \hat{Y}^l, \hat{B}^l, \hat{A}^l, \hat{M}^l) \to (\hat{X}, \hat{Y}, \hat{B}, \hat{A}, \hat{M})
  \end{equation*}
  in $C([0,T]; \bR^d \times \bR^d \times \bR^m \times \bR^d \times \bR^d)$ as $l \to \infty$, $\hat{\bP}$-a.s., and $\hat{\xi}^l \to \hat{\xi}$ as $l \to \infty$, $\hat{\bP}$-a.s.\footnote{One may drop the $\xi, \hat{\xi}^l$, $l \in \bN$, here, as they are uniquely determined by the $X$, $\hat{X}^l$, respectively.} With $\stackrel{\mathscr{D}}{\sim}$ we denote equality in law. From here on we identify any space of continuous functions with the supremum norm.

  Applying Fatou's Lemma, we obtain
  \begin{align*}
    \begin{split}
    \delta & \leq \liminf_{k \to \infty}
    \bE\bigg[ \sup_{t \in [0,T]} | X^{n_k}_t-X_t|^2\bigg]\\
    & \leq \liminf_{l \to \infty} \bE_{\hat{\bP}} \bigg[\sup_{t \in [0,T]} |\hat{Y}^l_t-\hat{X}^l_t|^2\bigg]\\
    & \leq \bE_{\hat{\bP}} \bigg[\limsup_{l \to \infty} \sup_{t \in [0,T]} |\hat{Y}^l_t-\hat{X}^l_t|^2\bigg]\\
    & = \bE_{\hat{\bP}} \bigg[\sup_{t \in [0,T]} |\hat{Y}_t-\hat{X}_t|^2\bigg].
    \end{split}
  \end{align*}
  We check that $(\hat{X},\hat{Z})$ with $\hat{Z}:=\hat{A}+\hat{M}$, $(\hat{\Omega},\hat{\mathcal{F}},\hat{\bP})$, $(\hat{\mathcal{F}}_t)_{t\in [0,T]}$ solves the Volterra local martingale problem \cite[Definition~2.4]{Promel2022} given $(\xi g,\mu,\sigma, K_\mu, K_\sigma)$. Then, by \cite[Lemma~2.7]{Promel2022} $\hat{Y}$ is a solution of the SVE~\eqref{eq:SVE}. Conditions (i)-(iii) of \cite[Definition~2.4]{Promel2022} are clear. (iv) of \cite[Definition~2.4]{Promel2022} follows as in the proof of \cite[Lemma~3.9]{Promel2022}. We only need to note that the definitions and results of aforementioned paper \cite{Promel2022} extend verbatim to the slightly more general present setting.

  To show (v) of \cite[Definition~2.4]{Promel2022} we introduce the processes
  \begin{equation*}
    Z^l := A^{n_{k_l}} + M^{n_{k_l}} \quad \text{and}\quad \hat{Z}^l = \hat{A}^l + \hat{M}^l,\quad l \in \bN.
  \end{equation*}
  Since $(\hat{Y}^l, \hat{M}^l) \stackrel{\mathscr{D}}{\sim} (X^{n_{k_l}},M^{n_{k_l}})$, for every $l \in \bN,$ and pathwise uniqueness holds by assumption a general version of the Yamada--Watanabe result (see, e.g., \cite{Kurtz2014}) shows that we may express $\hat{Y}^l$ as solution of
  \begin{equation}\label{eq:Y_as_Yama_Wata}
    \hat{Y}^l_t = \hat{\xi}^l g_{k_l}(t) + \int_0^t K_{\mu,n_{k_l}}(t-s) \mu_{n_{k_l}}(s,\hat{Y}_s^l) \dd s + \int_0^t K_{\sigma,n_{k_l}}(t-s) \dd \hat{M}^l_s, \qquad t \in [0,T].
  \end{equation}
  We know that $\hat{Y}^l \to \hat{Y}$ and that $\hat{\xi}^l g_{k_l} \to \hat{\xi} g$ $\hat{\bP}$-a.s. By $\hat{\xi}^l \stackrel{\mathscr{D}}{\sim} \xi$, $l \in \bN,$ we can conclude $\hat{\xi} \stackrel{\mathscr{D}}{\sim} \xi$. Next we show
  \begin{equation}\label{eq:conv_A}
    \Big(\int_0^t K_{\mu,n_{k_l}}(t-s) \dd \hat{A}^l_s \Big)_{t \in [0,T]} \stackrel{\hat{\bP}}{\longrightarrow} \Big(\int_0^t K_{\mu}(t-s) \dd \hat{A}_s \Big)_{t \in [0,T]}.
  \end{equation}
  Therefore, let $\eta, \phi>0$ be arbitrary but fixed. Denoting $\bar{K}:=\int_0^T |K_\mu(s)|\dd s$, we choose $N_1 \in \bN$ and $L_1 \in \bN$ sufficiently large, such that
  \begin{equation*}
    \hat{\bP}\Big(\|\hat{Y}\|_\infty \geq \frac{N_1}{2}\Big) \leq \frac{\phi}{3}, \qquad \hat{\bP}\Big(\|\hat{Y}^l-\hat{Y}\|_\infty \geq \max\Big(\frac{\eta}{3 C_{\mu,\sigma} \bar{K}}, \frac{N_1}{2}\Big)\Big) \leq \frac{\phi}{3}
  \end{equation*}
  for all $l \geq L_1$. On $\{\| \hat{Y}\|_\infty \vee \|\hat{Y}^l\|_\infty \leq N_1\}$, we have
  \begin{align*}
    &|G^l_t-G_t| \\
    &\quad:= \Big|\int_0^t K_{\mu,n_{k_l}}(t-s) \mu_{n_{k_l}}(s,\hat{Y}^l_s) \dd s - \int_0^t K_{\mu}(t-s) \mu_{n}(s,\hat{Y}^l_s) \dd s \Big| \\
    & \quad\leq \Big|\int_0^t (K_{\mu,n_{k_l}}(t-s) - K_{\mu}(t-s))\mu_{n_{k_l}}(s,\hat{Y}^l_s) \dd s \Big| \\
    & \quad\qquad + \int_0^t |K_\mu(s)| \dd s  \Big(\sup_{s \in [0,T]} \sup_{x \in [-N_1,N_1]} |\mu_{n_{k_l}}(s,x)-\mu(s,x)| + \sup_{s \in [0,T]} |\mu(s,\hat{Y}^l_s)-\mu(s,\hat{Y}_s)|\Big)\\
    & \quad\leq C_{\mu,\sigma} (1 + N_1) \int_0^t |K_{\mu,n_{k_l}}(s) - K_{\mu}(s)|\dd s\\
    &\quad \qquad + \bar{K} \Big( \sup_{s \in [0,T]} \sup_{x \in [-N_1,N_1]} |\mu_{n_{k_l}}(s,x)-\mu(s,x)| + C_{\mu,\sigma} \|\hat{Y}^l-\hat{Y}\|_\infty \Big).
  \end{align*}
  By the convergence of the kernels and coefficients there is an $L_2 \geq L_1$ such that, for all $l \geq L_2$,
  \begin{align*}
    C_{\mu,\sigma} (1 + N_1) \int_0^T | K_{\mu,n_{k_l}}(s) - K_\mu(s) | \dd s &\leq \frac{\eta}{3},\\
    \bar{K} \sup_{s \in [0,T]} \sup_{x \in [-N_1,N_1]} |\mu_{n_{k_l}}(s,x)-\mu(s,x)| & \leq \frac{\eta}{3}.
  \end{align*}
  Note that $\hat{\bP}(\bar{K} C_{\mu,\sigma} \|\hat{Y}^l- \hat{Y} \|_\infty \geq \frac{\eta}{3}) \leq \frac{\phi}{3}$ and
  \begin{align*}
    \hat{\bP} ( \| \hat{Y}\|_\infty \vee \| \hat{Y}^l\|_\infty \geq N_1) & \leq \hat{\bP} ( \{ \| \hat{Y}\|_\infty \geq N_1 \} \cup \{ \| \hat{Y}^l - \hat{Y}\|_\infty + \| \hat{Y}\|_\infty \geq N_1 \})\\
    & \leq \hat{\bP} \Big( \Big\{ \| \hat{Y}\|_\infty \geq \frac{N_1}{2} \Big\} \Big) + \hat{\bP} \Big( \| \hat{Y}^l - \hat{Y}\|_\infty \geq \frac{N_1}{2}\Big)\\
    & \leq \frac{2 \phi}{3}.
  \end{align*}
  Hence, for all $l \geq L_2$ we have
  \begin{align*}
    &\hat{\bP}(\|G^l-G\|_\infty \geq \eta)\\
    & \leq \hat{\bP} (\{\|G^l-G\|_\infty \geq \eta\} \cap \{ \| \hat{Y}\|_\infty \vee \| \hat{Y}^l\|_\infty < N_1\}) + \hat{\bP} ( \| \hat{Y}\|_\infty \vee \| \hat{Y}^l\|_\infty \geq N_1)\\
    & \leq \phi.
  \end{align*}
  It remains to show that
  \begin{equation*}
    \Big(\int_0^t K_{\sigma,n_{k_l}}(t-s) \dd \hat{M}^l_s \Big)_{t \in [0,T]} \stackrel{\hat{\bP}}{\longrightarrow} \Big(\int_0^t K_{\sigma}(t-s) \dd \hat{M}_s \Big)_{t \in [0,T]}.
  \end{equation*}
  As $\tilde{p}=\frac{p}{p-2}\leq 1 + \frac{\varepsilon}{2}$, using the Burkholder--Davis--Gundy inequality, we get, for any $t \in [0,T]$,
  \begin{align*}
    &\bE_{\hat{\bP}} \Big[ \Big( \int_0^t K_{\sigma,n_{k_l}} (t-s) \dd \hat{M}^l_s - \int_0^t K_{\sigma} (t-s) \dd \hat{M}_s\Big)^p \Big]^{\frac{1}{p}}\\
    & \quad\leq \bE_{\hat{\bP}} \Big[ \Big( \int_0^t K_{\sigma,n_{k_l}} (t-s) \dd (\hat{M}^l_s - \hat{M}_s)\Big)^p\Big]^{\frac{1}{p}} \\
    &\quad \qquad + \bE_{\hat{\bP}} \Big[ \Big( \int_0^t (K_{\sigma,n_{k_l}} (t-s) - K_{\sigma} (t-s)) \dd \hat{M}_s\Big)^p \Big]^{\frac{1}{p}}\\
    & \quad\leq \bE_{\hat{\bP}} \Big[ \Big( \int_0^t \big(K_{\sigma,n_{k_l}} (t-s) (\sigma_{n_{k_l}}(s, \hat{Y}^l_s) - \sigma(s, \hat{Y}_s))\big)^2 \dd s \Big)^{\frac{p}{2}} \Big]^{\frac{1}{p}}\\
    &\quad \qquad + \bE_{\hat{\bP}} \Big[ \Big( \int_0^t \big((K_\sigma(t-s) - K_{\sigma,n_{k_l}}(t-s)) \sigma(s,\hat{Y}_s)\big)^2 \dd s \Big)^{\frac{p}{2}} \Big]^{\frac{1}{p}}\\
    &\quad \leq C_{p,t} \Big(\Big( \int_0^t |K_{\sigma,n_{k_l}}(s)|^{2 \tilde{p}} \dd s \Big)^{\frac{p}{2 \tilde{p}}} \bE_{\hat{\bP}} \Big[ \int_0^t | \sigma_{n_{k_l}} (s, \hat{Y}^l_s)- \sigma (s, \hat{Y}_s)|^p \dd s \Big]^{\frac{1}{p}}\\
    &\quad \qquad + \Big( \int_0^t | K_\sigma(s) - K_{\sigma,n_{k_l}}(s)|^{2 \tilde{p}} \dd s \Big) ^{\frac{p}{2 \tilde{p}}} \bE_{\hat{\bP}} \Big[ \int_0^t | \sigma(s,\hat{Y}_s)|^p \dd s \Big]^{\frac{1}{p}} \Big).
  \end{align*}
  Note that $\int_0^t | K_{\sigma,n_{k_l}}(s)|^{2 \tilde{p}} \dd s$ is uniformly bounded in $l$. We can obtain
  \begin{equation*}
    \int_0^t \sigma_{n_{k_l}}(s,\hat{Y}^l_s) \dd s \stackrel{\hat{\bP}}{\longrightarrow} \int_0^t \sigma(s,\hat{Y}_s) \dd s
  \end{equation*}
  by the same steps we used to show \eqref{eq:conv_A}. One can mimic the proof of \cite[Lemma~3.4]{Promel2023} to obtain
	\begin{equation*}
		\sup_{t \in [0,T]} \bE_{\hat{\bP}}[| \hat{Y}^l_t|^p] \leq C_{p,L,\gamma,\epsilon,T,\mu,\sigma} \Big(1 + \bE_{\hat{\bP}}[|\hat{\xi}^l|^p] \sup_{t \in [0,T]} |g_{k_l}(t)|^p\Big)
	\end{equation*}
	where $C_{p,L,\gamma,\epsilon,T,\mu,\sigma}$ depends only on $p,L,\gamma,\epsilon,T$ and the linear growth constant $C_{\mu,\sigma}$ of the coefficients (see Assumption \ref{ass:coefficients2}). Since the $\hat{\xi}^l$, $l \in \bN$, are identically distributed with finite $p$-th moment, and since $\sup_{t \in [0,T]} |g_{k_l}(t)| \leq CT^\gamma +1$ by the $\gamma$-H{\"o}lder-continuity of the $g^l$, $l \in \bN$, we then get uniform $p$-integrability of $\hat{Y}^l$. Together with the uniform linear growth condition on $\sigma_{n_{k_l}}$, $l \in \bN$, one gets
  \begin{equation*}
    \bE_{\hat{\bP}} \Big[ \int_0^t | \sigma_{n_{k_l}} (s, \hat{Y}^l_s)- \sigma (s, \hat{Y}_s)|^p \dd s \Big] \to 0 \qquad \text{ as } l \to \infty.
  \end{equation*}
  Therefore, with Assumption~\ref{ass:kernel} we can conclude that
  \begin{equation*}
    \bE_{\hat{\bP}} \Big[ \Big( \int_0^t K_{\sigma,n_{k_l}} (t-s) \dd \hat{M}^l_s - \int_0^t K_{\sigma} (t-s) \dd \hat{M}_s\Big)^p \Big]^{\frac{1}{p}} \to 0 \qquad \text{ as } l \to \infty
  \end{equation*}
  and it follows that, for all $t \in [0,T]$,
  \begin{equation*}
    \int_0^t K_{\sigma,n_{k_l}}(t-s) \dd \hat{M}^l_s  \stackrel{\hat{\bP}}{\longrightarrow} \int_0^t K_{\sigma}(t-s) \dd \hat{M}_s \qquad \text{ as } l \to \infty.
  \end{equation*}
  By \eqref{eq:Y_as_Yama_Wata} we know that there is some continuous process $\hat{V}= (\hat{V}_t)_{t \in [0,T]}$ such that
  \begin{equation*}
    \sup_{r \in [0,T]} \Big| \int_0^r K_{\sigma,n_{k_l}}(r-s) \dd \hat{M}^l_s - \hat{V}_r \Big| \stackrel{\hat{\bP}}{\longrightarrow} 0 \qquad \text{ as } l \to \infty.
  \end{equation*}
  Using a uniqueness of limits argument, one obtains $\hat{V}_t = \int_0^t K_\sigma(t-s) \dd \hat{M}_s$ for all $t \in [0,T]$ and by the continuity we can conclude that the processes are indistinguishable. Taking a limit in probability in \eqref{eq:Y_as_Yama_Wata} or the $\hat{\bP}$-a.s. limit of some subsequence, we obtain that $\hat{X}$ is a solution to the SVE~\eqref{eq:SVE} such that $\E [\sup_{t\in [0,T]}|\hat{X}_t-X_t|^2]\geq \delta$, which contradicts the assumption that pathwise uniqueness holds for SVE~\eqref{eq:SVE} and, thus, completes the proof.
\end{proof}

\begin{proof}[Proof of Proposition~\ref{prop: stability SVE}]
  First notice that, due to Assumption~\ref{ass:norms} and Assumption~\ref{ass:coefficients}, there exist unique solutions $(X_t)_{t\in[0,T]}$ and $(\tilde{X}_t)_{t\in[0,T]}$ to the SVEs \eqref{eq:SVE} and \eqref{eq:SVE2}, see \cite[Theorem~1.1]{Wang2008}.

  Let $t\in [0,T]$ and $C>0$ be a generic constant, which may change from line to line. We get that
  \begin{align*}
    \E\big[ |X_t-\tilde{X}_t|^p \big]
    &= \E\bigg[\bigg| \xi\big(g(t)-\tilde{g}(t)\big) + \int_0^t K_\mu(t-s)\mu(s,X_s)\dd s - \int_0^t \tilde{K}_\mu(t-s)\tilde{\mu}(s,\tilde{X}_s)\dd s\\
    &\quad + \int_0^t K_\sigma(t-s)\sigma(s,X_s)\dd B_s - \int_0^t \tilde{K}_\sigma(t-s)\tilde{\sigma}(s,\tilde{X}_s)\dd B_s \bigg|^p\bigg]\\
    &\leq C \bigg(\sup_{s\in[0,T]}\big|g(s)-\tilde{g}(s)\big|^p + \E\Big[\Big|\int_0^t \big(K_\mu(t-s)-\tilde{K}_\mu(t-s)\big)\mu(s,X_s)\dd s\Big|^p\Big] \\
    &\qquad+\E\Big[\Big| \int_0^t \tilde{K}_\mu(t-s)\big(\mu(s,X_s)-\tilde{\mu}(s,\tilde{X}_s)\big)\dd s\Big|^p\Big]\\
    &\qquad + \E\Big[\Big|\int_0^t \big(K_\sigma(t-s)-\tilde{K}_\sigma(t-s)\big)\sigma(s,X_s)\dd B_s\Big|^p\Big]\\
    &\qquad + \E\Big[\Big| \int_0^t \tilde{K}_\sigma(t-s)\big(\sigma(s,X_s)-\tilde{\sigma}(s,\tilde{X}_s)\big)\dd B_s\Big|^p\Big] \bigg).
  \end{align*}
  Applying the Burkholder--Davis--Gundy inequality and H{\"o}lder's inequality with \eqref{def:p}, we deduce that
  \begin{align*}
    \E\big[ |X_t-\tilde{X}_t|^p \big]
    &\leq C \bigg(\|g-\tilde{g}\|_{\infty}^p + \Big(\int_0^t \big|K_\mu(t-s)-\tilde{K}_\mu(t-s)\big|^q\dd s\Big)^{\frac{p}{q}}\E\big[\int_0^t\big|\mu(s,X_s)\big|^p\dd s\big] \\\
    &\qquad+\Big(\int_0^t \big|\tilde{K}_\mu(t-s)\big|^q\dd s\Big)^{\frac{p}{q}}\E\big[\int_0^t \big|\mu(s,X_s)-\tilde{\mu}(s,\tilde{X}_s)\big|^p\dd s\big]\\
    &\qquad + \E\Big[\Big|\int_0^t \big(K_\sigma(t-s)-\tilde{K}_\sigma(t-s)\big)^2\sigma(s,X_s)^2\dd s\Big|^{\frac{p}{2}}\Big]\\
    &\qquad + \E\Big[\Big| \int_0^t \tilde{K}_\sigma(t-s)^2\big(\sigma(s,X_s)-\tilde{\sigma}(s,\tilde{X}_s)\big)^2\dd s\Big|^{\frac{p}{2}}\Big] \bigg)\\
    &\leq C \bigg(\|g-\tilde{g}\|_{\infty}^p + \|K_\mu-\tilde{K}_\mu\|_q^p\int_0^t\big(1+\E\big[\big|X_s\big|^p\big]\big)\dd s\\
    &\qquad+\Big(\int_0^t \big|\tilde{K}_\mu(t-s)\big|^q\dd s\Big)^{\frac{p}{q}}\int_0^t \E\big[\big|\mu(s,X_s)-\mu(s,\tilde{X}_s)\big|^p\big]\dd s\notag\\
    &\qquad + \Big(\int_0^t \big|K_\sigma(t-s)-\tilde{K}_\sigma(t-s)\big|^{2\tilde{q}}\dd s\Big)^{\frac{p}{2\tilde{q}}}\E\big[\int_0^t|\sigma(s,X_s)|^p\dd s\big]\notag\\
    &\qquad +\Big( \int_0^t \big|\tilde{K}_\sigma(t-s)\big|^{2\tilde{q}}\dd s\Big)^{\frac{p}{2\tilde{q}}} \E\Big[\int_0^t\big|\sigma(s,X_s)-\tilde{\sigma}(s,\tilde{X}_s)\big|^p\dd s\Big] \bigg).
  \end{align*}
  Using the regularity assumptions on $\mu$ and $\sigma$ (Assumption~\ref{ass:coefficients}) and the boundedness of all moments of Volterra processes (see \cite[Lemma~3.4]{Promel2023}), we get
  \begin{align*}
    \E\big[ |X_t-\tilde{X}_t|^p \big]
    \leq C &\bigg(\|g-\tilde{g}\|_{\infty}^p+\|\mu-\tilde{\mu}\|_{\infty}^p+\|\sigma-\tilde{\sigma}\|_{\infty}^p+\|K_\mu-\tilde{K}_\mu\|_{q}^p + \|K_\sigma -\tilde{K}_\sigma\|_{2\tilde{q}}^p\bigg)\\
    &\quad +C \int_0^t \E\big[\big|X_s-\tilde{X}_s\big|^p\big]\dd s.
  \end{align*}
  Applying Gr{\"o}nwall's lemma leads to
  \begin{align*}
    \E\big[ |X_t-\tilde{X}_t|^p \big]
    \leq C &\bigg(\|g-\tilde{g}\|_{\infty}^p+\|\mu-\tilde{\mu}\|_{\infty}^p+\|\sigma-\tilde{\sigma}\|_{\infty}^p+\|K_\mu-\tilde{K}_\mu\|_{q}^p + \|K_\sigma -\tilde{K}_\sigma\|_{2\tilde{q}}^p\bigg),
  \end{align*}
  which implies \eqref{eq:lem_stability} by taking the supremum on the left-hand side.
\end{proof}

\bibliography{literature}{}
\bibliographystyle{amsalpha}

\end{document}